\documentclass[11pt]{article}

\usepackage{amssymb,amsmath,amsthm}
\usepackage{color}
\usepackage{graphicx}
\usepackage{tikz}
\usetikzlibrary{shapes.geometric}

\def\oneb{{\bf 1}}

\def\a{{\bf a}}
\def\b{{\bf b}}
\def\c{{\bf c}}

\def\eb{{\bf e}}
\def\f{{\bf f}}
\def\gbold{{\bf g}}

\def\rbold{{\bf r}}

\def\v{{\bf v}}

\def\y{{\bf y}}
\def\z{{\bf z}}

\def\B{{\cal B}}
\def\C{{\cal C}}

\def\F{{\cal F}}

\def\K{{\cal K}}
\def\KR{{\cal KR}}

\def\U{{\cal U}}

\def\X{{\cal X}}
\def\Y{{\cal Y}}

\def\R{{\mathbb R}}

\def\Sm{{\mathbb S}}

\def\al{\alpha}
\def\d{\delta}

\def\e{\epsilon}
\def\g{\gamma}

\def\l{\lambda}

\def\om{\omega}
\def\OM{\Omega}

\def\s{\sigma}

\def\t{\tau}
\def\th{\theta}

\def\balpha{{\boldsymbol \alpha}}

\def\bmu{{\boldsymbol \mu}}

\def\bth{{\boldsymbol \theta}}
\def\bphi{{\boldsymbol \phi}}

\def\bxi{{\boldsymbol \xi}}

\def\bths{\bth^*}

\def\vb{\bar{v}}
\def\vh{\hat{v}}
\def\Vh{\hat{V}}

\def\tai{t \ap \infty}

\def\ap{\rightarrow}

\def\bz{{\bf 0}}

\def\fa{\; \forall}

\def\as{\mbox{ a.s.}}

\def\nm{\Vert}

\renewcommand{\and}{\mbox{$\wedge$}}

\newcommand{\bc}{\begin{center}}
\newcommand{\ec}{\end{center}}
\newcommand{\be}{\begin{equation}}
\newcommand{\ee}{\end{equation}}
\newcommand{\bd}{\begin{displaymath}}
\newcommand{\ed}{\end{displaymath}}
\newcommand{\ba}{\begin{array}}
\newcommand{\ea}{\end{array}}
\newcommand{\ben}{\begin{enumerate}}
\newcommand{\een}{\end{enumerate}}
\newcommand{\bit}{\begin{itemize}}
\newcommand{\eit}{\end{itemize}}
\newcommand{\beq}{\begin{eqnarray}}
\newcommand{\eeq}{\end{eqnarray}}
\newcommand{\btab}{\begin{tabular}}
\newcommand{\etab}{\end{tabular}}
\newcommand{\bfig}{\begin{figure}}
\newcommand{\efig}{\end{figure}}
\newcommand{\btp}{\begin{tikzpicture}}
\newcommand{\etp}{\end{tikzpicture}}

\newcommand{\argmin}{\operatornamewithlimits{arg~min}}
\newcommand{\argmax}{\operatornamewithlimits{arg~max}}

\newcommand{\nmm}[1]{ \nm #1 \nm }
\newcommand{\nmeu}[1]{ \nm #1 \nm_2 }
\newcommand{\nmeusq}[1]{ \nm #1 \nm_2^2 }
\newcommand{\nmi}[1]{ \nm #1 \nm_\infty}

\newcommand{\nmS}[1]{ \nm #1 \nm_S }

\def\gJ{\nabla J}

\newcommand{\IP}[2]{ \langle #1 , #2 \rangle }

\newcommand{\halmos}{\hfill $\blacksquare$}

\def\nmsl1{\nm_{{\rm SL1}}}

\definecolor{verm}{rgb}{0.6,0.2,0.2}
\definecolor{purp}{rgb}{0.3,0.1,0.6}
\definecolor{purple}{rgb}{0.4,0.0,0.6}
\definecolor{bggreen}{rgb}{0.1,0.3,0.1}
\definecolor{dgreen}{rgb}{0.1,0.6,0.1}
\definecolor{black}{rgb}{0.0,0.0,0.0}
\definecolor{crim}{rgb}{0.3,0.1,0.1}
\definecolor{dred}{rgb}{0.5,0.1,0.1}

\definecolor{Blue}{cmyk}{0.65,0.13,0,0}
\definecolor{Black}{cmyk}{0,0,0,1}
\definecolor{Red}{cmyk}{0,1,1,0}
\definecolor{Green}{cmyk}{1,0,1,0}
\definecolor{Orange}{cmyk}{0,0.61,0.87,0.1}
\definecolor{Fuchsia}{cmyk}{0.47,0.91,0,0.08}
\definecolor{PineGreen}{cmyk}{0.92,0,0.59,0.25}

{\bf}{\it}
\newtheorem{definition}{Definition}{\bf}{\it}
\newtheorem{example}{Example}{\bf}{\rm}
\newtheorem{lemma}{Lemma}{\bf}{\it}
\newtheorem{theorem}{Theorem}{\bf}{\it}
{\bf}{\it}
{\bf}{\it}
{\bf}{\rm}

\def\Ah{\hat{A}}

\def\Vd{\dot{V}}
\def\gJ{\nabla J}
\def\Tl{T^{\l}}
\def\TDl{TD(\l)}

\def\vbb{\bar{\v}}
\def\vbh{\hat{\v}}
\def\Vh{\hat{V}}

\begin{document}

\title{A Tutorial Introduction to Reinforcement Learning}

\author{
Mathukumalli Vidyasagar
\thanks{
SERB National Science Chair, Indian Institute of Technology Hyderabad,
Kandi, Telangana 502284, India.
Email: m.vidyasagar@iith.ac.in
This research was supported by the Science and Engineering
Research Board, Government of India.
}
}

\maketitle

\begin{abstract}

In this paper, we present a brief survey of Reinforcement Learning (RL),
with particular emphasis on Stochastic Approximation (SA) as a unifying theme.
The scope of the paper includes Markov Reward Processes, Markov Decision
Processes, Stochastic Approximation algorithms, and widely used
algorithms such as Temporal Difference Learning and $Q$-learning.

\end{abstract}

\section{Introduction}\label{sec:Intro}

In this paper, we present a brief survey of Reinforcement Learning (RL),
with particular emphasis on Stochastic Approximation (SA) as a unifying theme.
The scope of the paper includes Markov Reward Processes, Markov Decision
Processes, Stochastic Approximation methods, and widely used
algorithms such as
Temporal Difference Learning and $Q$-learning.
Reinforcement Learning is a vast subject, and this brief survey can
barely do justice to the topic.
There are several excellent texts on RL, such as
\cite{Ber-Tsi96,Puterman05,Csaba10,Sutton-Barto18}.
The dynamics of the Stochastic Approximation (SA) algorithm are
analyzed in \cite{Ljung78,Kushner-Clark78,BMP92,Kushner-Yin97,Benaim99,Borkar08,Borkar22}.
The interested reader may consult those sources for more information.

In this survey, we use the phrase ``reinforcement learning'' 
to refer to decision-making
with uncertain models, \textit{and in addition, current actions
alter the future behavior of the system}.
Therefore, if the same action
is taken at a future time, the consequences might not be the same.
This additional feature
distinguishes RL from ``mere'' decision-making under uncertainty.
Figure \ref{fig:11} rather arbitrarily divides decision-making problems
into four quadrants.
Examples from each quadrant are now  briefly described.
\bit
\item Many if not most decision-making problems fall into the lower-left
quadrant of ``good model, no alteration'' (meaning that the control
actions do not alter the environment).
An example  is a fighter aircraft which usually
has an excellent model thanks to aerodynamical modelling and/or wind tunnel
tests.
In turn this permits the control system designers to formulate and to solve
an optimal (or some other form of) control problem.
\item Controlling a chemical reactor would be an example from the lower-right
quadrant.
As a traditional control system, it can be assumed that the environment
in which the reactor operates does not change as a consequence of
the control strategy adopted.
However, due to the complexity of a reactor, it is difficult to obtain
a very accurate model, in contrast with a fighter aircraft for example.
In such a case, one can adopt one of two approaches.
The first, which is a traditional approach in control system theory,
is to use a nominal model of the system and to treat the deviations
from the nominal model as uncertainties in the model.
A controller is designed based on the nominal model, and robust control
theory would be invoked to ensure that the controller would still
perform satisfactorily (though not necessarily optimally) for the actual
system.
The second, which would move the problem from the lower right to the
upper right quadrant, is to attempt to ``learn'' the unknown dynamical
model by probing its response to various inputs.
This approach is suggested in \cite[Example 3.1]{Sutton-Barto18}.
A similar statement can be made about robots, where the geometry determines
the \textit{form} of the dynamical equations describing it, but not
the parameters in the equations; see for example \cite{SHV20}.
In this case too, it is possible to ``learn'' the dynamics through
experimentation.
In practice, such an approach is far slower than the traditional
control systems approach of using a nominal model and designing a
``robust'' controller.
However, ``learning control'' is a popular area in the world of machine
learning.
One reason is that the initial modelling error is too large, then
robust control theory alone would not be sufficient to ensure the
stability of the actual system with the designed controller.
In contrast (and in principle), a ``learning control'' approach
can withstand larger modelling errors.
The widely-used Model Predictive Control (MPC) paradigm can be viewed
as an example of a learning-based approach.
\item A classic example of a problem belonging to the upper-left corner
is a Markov Decision Process (MDP),
which forms the backbone of one approach to RL.
In an MDP, there is a state space $\X$, and an action space $\U$,
both of which are usually assumed to be finite.
In most MDPs, $|\X| \gg |\U|$.
Board games without an element of randomness such as tic-tac-toe or chess
would belong to the upper-left quadrant, at least in principle.
Tic-tac-toe belongs here, because the rules of the game are
clear, and the number of possible games is manageable.
\textit{In principle}, games such as chess which are ``deterministic'' (i.e.,
there is no throwing of dice as in Backgammon for example) would also
belong here.
Chess is a two-person game in which, for each board position, it is possible
to assign the likelihood of the three possible outcomes:
White wins, Black wins, or it is a draw.
However, due to the enormous number of possibilities, it is often not possible
to \textit{determine} these likelihoods precisely.
It is pointed out explicitly in \cite{Shannon-chess} that, merely because
we cannot explicitly compute this likelihood function, that does not
mean that the likelihood does not exist!
However, as a practical matter, it is not a bad idea to treat this
likelihood function as being unknown, and to \textit{infer} it on the basis of
experiment / experience.
Thus, as with chemical reactors, it is not uncommon to move chess-playing
from the lower-right corner to the upper-right corner.
\item The upper-right quadrant is the focus of RL.
There are many possible ways to formulate RL problems, each of which 
leads to its own solution methodologies.
A very popular approach to RL is to formulate it as MDPs whose dynamics
are unknown.
That is the approach adopted in this paper.
\eit

\bfig
\bc
\btp[every text node part/.style={align=center}, line width = 2pt]

\draw [->] (0,0) -- (8.5,0) node [right] {Model Quality} ;
\draw [->] (0,0) -- (0,4.5) node [above]
{Interaction Between \\ Action \& Environment} ;
\draw [blue] (0.1,0.1) rectangle node [pos=0.5] 
{Good model. \\ Action doesn't alter \\ Environment} (3.9,1.9) ;
\draw [purp] (0.1,2.1) rectangle node [pos=0.5] 
{Good model. \\ Action can alter \\ Environment} (3.9,3.9) ;
\draw [bggreen] (4.1,0.1) rectangle node [pos=0.5] 
{Poor model. \\ Action doesn't alter \\ Environment} (7.9,1.9) ;
\draw [red] (4.1,2.1) rectangle node [pos=0.5] 
{Poor model. \\ Action can alter \\ Environment} (7.9,3.9) ;

\etp
\ec
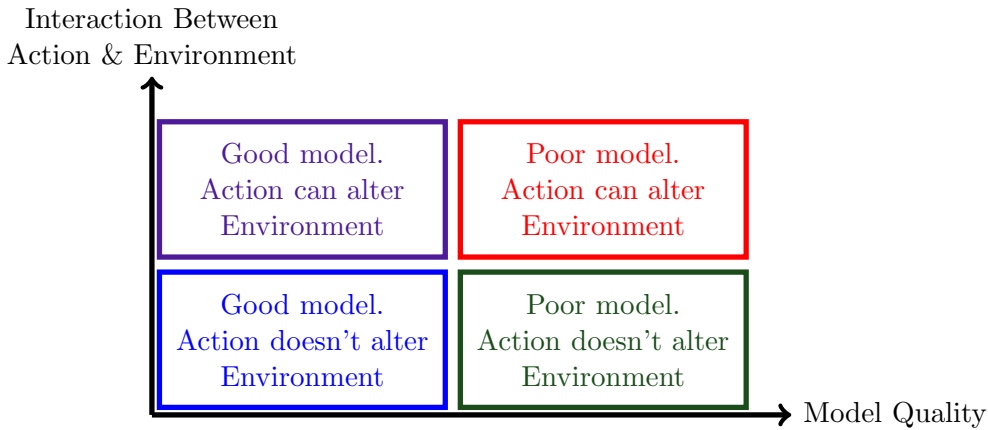
\caption{The four quadrants of decision-making under uncertainty}
\label{fig:11}
\efig

In an MDP,
at each time $t$, the learner (also known as the actor or the agent)
measures the state $X_t \in \X$.
Based on this measurement, the learner chooses an action $U_t \in \U$,
and receives a reward $R(X_t,U_t)$.
Future rewards are discounted by a discount factor $\g \in (0,1)$.
The rule by which the current action $U_t$ is chosen as a function of
the current state $X_t$ is known as a policy.
With each policy, one can associate the expected value of the total
(discounted) reward over time.
The problem is to find the best policy.
There is a variant called POMDP (Partially Observable Markov Decision Process)
in which the state $X_t$ cannot be measured directly;
rather, there is an output (or observation) $Y_t \in \Y$ which is
a memoryless function, either deterministic or random, of $X_t$.
These problems are not studied here; it is always assumed that $X_t$
can be measured directly.
When the parameters of the MDP are known, there are several approaches
to determining the optimal policy.
RL is distinct from an MDP in that, in RL, the parameters of the underlying
MDP are constant but not known to the learner; they must be learnt on the 
basis of experimentation.
Figure \ref{fig:12} depicts the situation.

\bfig
\bc
\btp[line width = 2pt]


\draw (-3,0) rectangle (3,1.2)
node (MC) [pos=0.5,align=center]
{{\color{verm}Environment:Markov Process} \\
{\color{blue}States $\{X_0,X_1, \cdots \}$}} ;

\draw (-1.5,-3.2) rectangle (1.5,-2)
node(pol) [pos=0.5,align=center]
{{\color{verm}Policy $\pi$} \\
{\color{blue}$U_t = \pi(X_t)$} } ;

\draw (4,-1.6) rectangle (7,-0.4)
node(rew) [pos=0.5,align=center]
{{\color{verm} Reward $R$} \\
{\color{blue}$R_t = R(X_t,U_t)$} } ;


\draw [->] (0,0) -- node [left] {{\color{blue}$X_t$}} (0,-2) ;
\draw [->] (-1.5,-2.6) -- node [above] {{\color{blue}$U_t$}}
(-3.5,-2.6) -- (-3.5,0.6) -- (-3,0.6) ;
\draw [->] (0,-1) -- node [above] {{\color{blue}$X_t$}} (4,-1) ;
\draw [->] (1.5,-2.6) -- node [above] {{\color{blue}$U_t$}}
(5.5,-2.6) -- (5.5,-1.6) ;

\etp
\ec
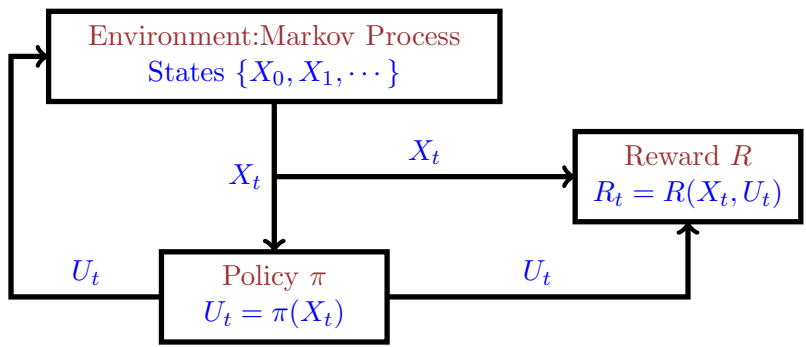
\caption{Depiction of a Reinforcement Learning Problem}
\label{fig:12}
\efig

The remainder of the paper is organized as follows:
In Section \ref{sec:MRP}, Markov Reward Processes are introduced.
These are a precursor to Markov Decision Processes (MDPs),
which are introduced in Section \ref{sec:MDP}.
Specifically, in Section \ref{ssec:31}, the relevant problems
in the study of MDPs are formulated.
In Section \ref{ssec:32}, the solutions to these problems are given
in terms of the Bellman value iteration, the action-value function,
and the $F$-iteration to determine the optimal action-value function.
In Section \ref{ssec:33}, we study the situation where the dynamics
of the MDP under study are not known precisely.
Instead, one has access only to a sample path $\{ X_t \}$ of the
Markov process under study.
For this situation, we present two standard algorithms, known as
Temporal Difference Learning, and $Q$-Learning.
Starting
from Section \ref{sec:SA}, the paper consists of results due to the author.
In Section \ref{ssec:41}, the concept of stochastic approximation (SA)
is introduced, and its relevance to Reinforcement Learning is outlined
in Section \ref{ssec:42}.
In Section \ref{ssec:43}, a new theorem on the global asymptotic
stability of nonlinear ODEs is stated; this theorem is of independent interest.
Some theorems on the convergence of the SA algorithm are
presented in Sections \ref{ssec:44} and \ref{ssec:45}.
In Section \ref{sec:Appl}, the results of Section \ref{sec:SA}
are applied to RL problems.
In Section \ref{ssec:51}, a technical result on the sample paths of
an irreducible Markov process is stated.
Using this result, simplified conditions are given for the convergence
of the Temporal Difference algorithm (Section \ref{ssec:52})
and $Q$-learning (Section \ref{ssec:53}).
A brief set of concluding remarks ends the paper.


\section{Markov Reward Processes}\label{sec:MRP}

Markov reward processes are standard (stationary) Markov processes
where each state has a ``reward'' associated with it.
Markov Reward Processes are a precursor to Markov Decision Processes;
so we review those in this section.
There are several standard texts on Markov processes, one of which is
\cite{MV-HMM-14}.

Suppose $\X$ is a finite set of cardinality $n$, written as
$\{ x_1 , \ldots , x_n \}$.
If $\{ X_t \}_{t \geq 0}$ is a stationary Markov process assuming values
in $\X$, then the corresponding state transition matrix $A$ is defined by
\be\label{eq:S11}
a_{ij} = \Pr \{ X_{t+1} = x_j | X_t = x_i \} .
\ee
Thus the $i$-th row of $A$ is the conditional probability vector of $X_{t+1}$
when $X_t = x_i$.
Clearly the row sums of the matrix $A$ are all equal to one.
This can be expressed as $A \oneb_n = \oneb_n$, where $\oneb_n$
denotes the $n$-dimensional column vector whose entries all equal one.
Therefore, if we define the induced matrix norm $\nmm{A}_{\infty \ap \infty}$ as
\bd
\nmm{A}_{\infty \ap \infty} := \max_{\v \neq \bz} 
\frac{\nmi{A\v}} {\nmi{\v}},
\ed
then $\nmm{A}_{\infty \ap \infty}$ equals one,
which also equals the spectral radius of $A$.

Now suppose that there is a ``reward'' function $R: \X \ap \R$ associated
with each state.
There is no consensus within the community about whether the reward
corresponding to the state $X_t$ is paid at time $t$ as in \cite{Csaba10},
or time $t+1$, as in \cite{Puterman05,Sutton-Barto18}.
In this paper, it is assumed
that the reward is paid at time $t$, and is denoted by $R_t$;
the modifications required to handle the other approach are easy and
left to the reader.
The reward $R_t$ can either be a deterministic function of $X_t$,
or a random  function.
If $R_t$ is a deterministic function of $X_t$, then we have that
$R_t = R(X_t)$ where $R$ is the reward function mapping $\X$ into 
(a finite subset of) $\R$.
Thus, whenever the trajectory $\{ X_t \}$ of the Markov process equals
some state $x_i \in \X$, the resulting reward $R(X_t)$ will always
equal $R(x_i) =: r_i$.
Thus the reward is captured by an $n$-dimensional vector $\rbold$,
where $r_i = R(x_i)$.
On the other hand, if $R_t$ is a random function of $X_t$, then one
would have to provide the probability distribution of $R_t$ given $X_t$.
Since $X_t$ has only $n$ different values, we would have to provide
$n$ different probability distributions.
To avoid technical difficulties, it is common to assume that $R(x_i)$
is a \textit{bounded} random variable for each index $i$.
Note that, because the set $\X$ is finite, if the reward function is
deterministic, then we have that
\bd
\max_{x_i \in \X} R(x_i) < \infty .
\ed
In case the reward function $R$ is random, as mentioned above, it is
common to assume that $R(x_i)$ is a \textit{bounded} random variable
for each index $i \in [n]$, where the symbol $[n]$ equals
$\{ 1 , \cdots , n \}$.
With this assumption, it follows that
\bd
\max_{x_i \in \X} E[R(x_i)] < \infty .
\ed

Two kinds of Markov reward processes are widely studied, namely:
Discounted reward processes, and average reward processes.
In this paper, we restrict attention to discounted reward processes.
However, we briefly introduce average reward processes.
Define (if it exists)
\bd
V(x_i) :=  \lim_{T \ap 0} \frac{1}{T+1}
E \left[ \sum_{t=0}^T R(X_t) | X_0 = x_i \right] .
\ed
An excellent review of average reward processes can be found in
\cite{Avg-MDP-Survey93}.

In each discounted Markov Reward Process,
there is a ``discount factor'' $\g \in (0,1)$.
This factor captures the extent to which future rewards are less valuable
than immediate rewards.
Fix an initial state $x_i \in \X$.
Then the \textbf{expected discounted future reward} $V(x_i)$ is defined as
\be\label{eq:S12}
V(x_i) := E \left[ \sum_{t=0}^\infty \g^t R_t | X_0 = x_i \right] 
= E \left[ \sum_{t=0}^\infty \g^t R(X_t) | X_0 = x_i \right] .
\ee
We often just use ``discounted reward'' instead of the longer phrase.
With these assumptions, because $\g < 1$, the above summation converges and is
well-defined.
The quantity $V(x_i)$ is referred to as the \textbf{value function}
associated with $x_i$, and the vector
\be\label{eq:S13}
\v = [ \ba{lll} V(x_1) & \cdots & V(x_n) \ea ]^\top ,
\ee
is referred to as the \textbf{value vector}.
Note that, throughout this paper, we view the value as both a
\textit{function} $V : \X \ap \R$ as well as a \textit{vector} $\v \in \R^n$.
The relationship between the two is given by \eqref{eq:S13}.
We shall use whichever interpretation is more convenient in a given context.

This raises the question as to how the value function and/or value vector
is to be determined.
Define the vector $\rbold \in \R^n$ via
\be\label{eq:S14}
\rbold := [ \ba{lll} r_1 & \cdots & r_n \ea ]^\top ,
\ee
where $r_i = R(x_i)$ if $R$ is a deterministic function, and
if $R_t$ is a random function of $X_t$, then
\be\label{eq:S14a}
r_i := E[R(x_i)] .
\ee
The next result gives a useful characterization of the value vector.

\begin{theorem}\label{thm:S11}
The vector $\v$ satisfies the recursive relationship
\be\label{eq:S15}
\v = \rbold + \g A \v ,
\ee
or, in expanded form,
\be\label{eq:S15f}
V(x_i) = r_i + \g \sum_{j=1}^n a_{ij} V(x_j) .
\ee
\end{theorem}

\begin{proof}
Let $x_i \in \X$ be arbitrary.
Then by definition we have
\be\label{eq:S15a}
V(x_i) = E \left[ \sum_{t=0}^\infty \g^t R_t | X_0 = x_i \right] 
= r_i + E \left[ \sum_{t=1}^\infty \g^t R_t | X_0 = x_i \right] .
\ee
However, if $X_0 = x_i$, then $X_1 = x_j$ with probability $a_{ij}$.
Therefore we can write
\beq
E \left[ \sum_{t=1}^\infty \g^t R_t | X_0 = x_i \right]
& = & \sum_{j=1}^n a_{ij}
E \left[ \sum_{t=1}^\infty \g^t R_t | X_1 = x_j \right] \nonumber \\
& = & \g \sum_{j=1}^n a_{ij}
E \left[ \sum_{t=0}^\infty \g^t R_t | X_0 = x_j \right] \nonumber \\
& = & \g \sum_{j=1}^n a_{ij} V(x_j) . \label{eq:S15b}
\eeq
In the second step we use fact that the Markov process is stationary.
Substituting from \eqref{eq:S15b} into \eqref{eq:S15a} gives the
recursive relationship \eqref{eq:S15f}.
\end{proof}

\begin{example}\label{exam:S11}
As an illustration of a Markov Reward process,
we analyze a toy snakes and ladders game with the transitions shown in
Figure \ref{fig:3}.
Here $W$ and $L$ denote ``win'' and ``lose'' respectively.
The rules of the game are as follows:
\bit
\item Initial state is S.
\item A four-sided, fair die is thrown at each stage.
\item Player must land exactly on $W$ to win and exactly on $L$ to lose.
\item If implementing a move causes crossing of $W$ and $L$,
then the move is not implemented.
\eit
There are twelve possible states in all: S, 1, $\ldots$ , 9 , $W$, $L$.
However, 2, 3, 9 can be omitted, leaving nine states,
namely S, 1, 4, 5, 6, 7, 8, $W$, $L$.
At each step, there are at most four possible outcomes.
For example, from the state S, the four outcomes are 1, 7, 5, 4.
From state 6, the four outcomes are 7, 8, 1, and W.
From state 7, the four outcomes are 8, 1, W, 7.
From state 8, there four possible outcomes are 1, $W$, $L$ and 8 with
probability $1/4$ each, because if the die comes up with 4, then the move
cannot be implemented.
It is time-consuming but straight-forward to compute the state transition
matrix as
\bd
\ba{|c|ccccccc|cc|}
\hline
& S & 1 & 4 & 5 & 6 & 7 & 8 & W & L \\
\hline
S &      0 &  0.25 &  0.25 &  0.25 &     0 &  0.25 &     0  &    0 & 0 \\
1 &      0 &     0 &  0.25 &  0.50 &     0 &  0.25 &     0  &    0 & 0 \\
4 &      0 &     0 &     0 &  0.25 &  0.25 &  0.25 &  0.25  &    0 & 0 \\
5 &      0 &  0.25 &     0 &     0 &  0.25 &  0.25 &  0.25  &    0 & 0 \\
6 &      0 &  0.25 &     0 &     0 &     0 &  0.25 &  0.25  & 0.25 & 0 \\
7 &      0 &  0.25 &     0 &     0 &     0 &  0 &  0.25  & 0.25 & 0.25 \\
8 &      0 &  0.25 &     0 &     0 &     0 &     0 &  0.25  & 0.25 & 0.25 \\
\hline
W &      0 &     0 &     0 &     0 &     0 &     0 &     0  & 1 & 0 \\
L & 0 &     0 &     0 &     0 &     0 &     0 &     0  & 0 & 1 \\
\hline
\ea
\ed

\bfig
\bc
\btp[line width=2pt]


\draw (-1,-0.5) rectangle (11,0.5) ;
\draw (0,-0.5) -- node [left] {S} (0,0.5) ;
\draw (1,-0.5) -- node [left] {1} (1,0.5) ;
\draw (2,-0.5) -- node [left] {2} (2,0.5) ;
\draw (3,-0.5) -- node [left] {3} (3,0.5) ;
\draw (4,-0.5) -- node [left] {4} (4,0.5) ;
\draw (5,-0.5) -- node [left] {5} (5,0.5) ;
\draw (6,-0.5) -- node [left] {6} (6,0.5) ;
\draw (7,-0.5) -- node [left] {7} (7,0.5) ;
\draw (8,-0.5) -- node [left] {8} (8,0.5) ;
\draw (9,-0.5) -- node [left] {9} (9,0.5) ;
\draw (10,-0.5) -- node [left] {W} (10,0.5) ;
\draw (11,-0.5) -- node [left] {L} (11,0.5) ;


\draw [->,bggreen] (2.5,-0.5) .. controls (3.16,-1) and (3.83,-1) .. (4.5,-0.5) ;
\draw [->,bggreen] (1.5,-0.5) .. controls (3.17,-2) and (4.83,-2) .. (6.5,-0.5) ;
\draw [->,red] (8.5,0.5) .. controls (5.83,2) and (3.17,2) .. (0.5,0.5) ;

\etp
\ec
\caption{A Toy Snakes and Ladders Game}
\label{fig:3}
\efig

We define a reward function for this problem, as follows:
We set $R_t = f(X_{t+1})$, where $f$ is defined as follows:
$f(W) = 5$, $f(L) = -2$, $f(x) = 0$ for all other states.
However, there is an expected reward \textit{depending on the state at the next
time instant}.
For example, if $X_0 = 6$, then the expected value of $R_0$ is $5/4$,
whereas if $X_0 = 7$ or $X_0 = 8$, then the expected value of $R_0$ is $3/4$.
\halmos
\end{example}

Now let us see how the implicit equation \eqref{eq:S15} can be solved
to determine the value vector $\v$.
Since the induced matrix norm $\nmm{A}_{\infty \ap \infty} = 1$
and $\g < 1$, it follows that the matrix $I - \g A$ is nonsingular.
Therefore, for every reward function $\rbold$, there is
a unique $\v$ that satisfies \eqref{eq:S15}.
In principle it is possible to deduce from \eqref{eq:S15} that
\be\label{eq:S15d}
\v = (I - \g A)^{-1} \rbold .
\ee
The difficulty wth this formula however is that in most actual
applications of Markov Decision Problems, the integer $n$ denoting
the size of the state space $\X$ is quite large.
Moreover, inverting a matrix has cubic complexity in the size of the
matrix.
Therefore it may not be practicable to invert the matrix $I - \g A$.
So we are forced to look for alternate approaches.
A feasible approach is provided by the Contraction Mapping Theorem.

\begin{theorem}\label{thm:S11a}
The map $\y \mapsto T\y := \rbold + \g A \y$ is monotone and is a contraction
with respect to the $\ell_\infty$-norm, with contraction constant $\g$.
Therefore, we can choose some vector $\y^0$ arbitrarily, and then define
\bd
\y^{i+1} = \rbold + \g A \y^i .
\ed
Then $\y^i$ converges to the value vector $\v$.
\end{theorem}

\begin{proof}
The first statement is that if $\y_1 \leq \y_2$ componentwise (and note
that the vectors $\y_1 , \y_2$ need not consist of only positive components),
then $T\y_1 \leq T\y_2$.
This is obvious from the fact that the matrix $A$ has only nonnegative
components, so that $A\y_1 \leq A\y_2$.
For the second statement, note that, because the matrix $A$ is row-stochastic,
the induced matrix norm $\nmm{A}_{\infty \ap \infty}$ is equal to one.
Therefore
\bd
\nmm{T\y_1 - T\y_2}_\infty = \nmm{ \g A(\y_1 - \y_2)}_\infty \leq
\g \nmm{\y_1 - \y_2}_\infty.
\ed
This completes the proof.
\end{proof}

There is however a limitation to this approach, namely, that it requires
that the state transition matrix $A$ has to be known.
In Reinforcement Learning, this assumption is often  not satisfied.
Instead, one has access to a single sample path $\{ X_t \}$ of a Markov
process over $\X$, whose state transition matrix is $A$.
The question therefore arises: How can one compute the value vector $\v$
in such a scenario?
The answer is provided by the so-called Temporal Difference algorithm,
which is discussed in Section \ref{ssec:33}.

\section{Markov Decision Processes}\label{sec:MDP}

\subsection{Problem Formulation}\label{ssec:31}

In a Markov reward process, the state $X_t$ evolves on its own, according to
a predetermined state transition matrix.
In contrast, in a MDP, there is also another variable called the ``action''
which affects the dynamics.
Specifically, in addition to the state space $\X$, there is also a finite
set of actions $\U$.
Each action $u_k \in \U$ leads to a corresponding state transition matrix
$A^{u_k} = [ a_{ij}^{u_k} ]$.
So at time $t$, if the state is $X_t$, and an action $U_t \in \U$ is applied,
then
\be\label{eq:S16}
\Pr \{ X_{t+1} = x_j | X_t = x_i , U_t = u_k \} = a_{ij}^{u_k} .
\ee
Obviously, for each fixed $u_k \in \U$, the corresponding state transition
matrix $A^{u_k}$ is row-stochastic.
In addition, there is also a ``reward'' function $R : \X \times \U \ap \R$.
Note that in a Markov reward process, the reward depends only on the
current state, whereas in a Markov decision process, the reward depends
on both the current state as well as the action taken.
As in Markov reward processes,
it is possible to permit $R$ to be a random function of $X_t$ and $U_t$
as opposed to a deterministic function.
Moreover, to be consistent with the earlier convention,
it is assumed that the reward $R(X_t,U_t)$ is paid at time $t$.

The most important aspect of an MDP is the concept of a ``policy,''
which is just a systematic way of choosing $U_t$ given $X_t$.
If $\pi: \X \ap \U$ is any map, this would be called a
\textit{deterministic} policy, and the set of all deterministic
policies is denoted by $\Pi_d$.
Alternatively, let $\Sm(\U)$ denote the set of probability distributions on the
finite set $\U$.
Then a map $\pi : \X \ap \Sm(\U)$ would be called a probabilistic policy,
and the set of probabilistic policies is denoted by $\Pi_p$.
Note that the cardinality of $\Pi_d$ equals $|\U||^{|\X|}$,
while the set $\Pi_p$ is uncountable.

A vital point about MDPs is this:
Whenever any policy $\pi$, whether deterministic or probabilistic, 
is implemented, the resulting process $\{ X_t \}$ is a Markov process
with an associated state transition matrix, which is denoted by $A_\pi$.
This matrix can be determined as follows:
If $\pi \in \Pi_d$, then at time $t$, if $X_t = x_i$, then the
corresponding action $U_t$ equals $\pi(x_i)$.
Therefore
\be\label{eq:S17}
\Pr \{ X_{t+1} = x_j | X_t = x_i , \pi \} = a_{ij}^{\pi(x_i)} .
\ee
If $\pi \in \Pi_p$ and
\be\label{eq:S18}
\pi(x_i) = [ \ba{lll} \phi_{i1} & \cdots & \phi_{im} \ea ] ,
\ee
where $m = |\U|$, then
\be\label{eq:S19}
\Pr \{ X_{t+1} = x_j | X_t = x_i , \pi \} =
\sum_{k=1}^m \phi_{ik} a_{ij}^{u_k} .
\ee
In a similar manner, for every policy $\pi$,
the reward function $R: \X \times \U \ap \R$ can
be converted into a reward map $R_\pi : \X \ap \R$, as follows:
If $\pi \in \Pi_d$, then
\be\label{eq:S110}
R_\pi(x_i) = R(x_i , \pi(x_i)) ,
\ee
whereas if $\pi \in \Pi_p$, then
\be\label{eq:S111}
R_\pi(x_i) = \sum_{k=1}^m \phi_{ik} R(x_i,u_k) .
\ee
Thus, given any policy $\pi$, whether deterministic or probabilistic,
we can associate with it a reward vector $\rbold_\pi$.
To summarize, given any MDP, once a policy $\pi$ is chosen,
the resulting process $\{ X_t \}$ is a Markov reward process with
state transition matrix $A_\pi$ and reward vector $\rbold_\pi$.

\begin{example}
To illustrate these ideas, suppose $n = 4 , m = 2$, so that there
four states and two actions.
Thus there are two $4 \times 4$
state transition matrices $A^1, A^2$ corresponding to the two actions.
(In the interests of clarity, we write $A^1$ and $A^2$ instead of
$A^{u_1}$ and $A^{u_2}$.)
Suppose $\pi_1$ is a deterministic policy, represented as a $n \times m$
matrix (in this  case a $4 \times 2$ matrix), as follows:
\bd
M_1 = \left[ \ba{cc} 
0 & 1 \\ 0 & 1 \\ 1 & 0 \\ 0 & 1 \ea \right] .
\ed
This means that if $X_t = x_1, x_2$ or $x_4$, then $U_t = u_2$,
while if $X_t = x_3$, then $U_t = u_1$.
Let us use the notation $(A^k)^i$ to denote the $i$-th row of the 
matrix $A^k$, where $k = 1, 2$ and $i = 1 , 2, 3, 4$.
Then the state transition matrix $A_{\pi_1}$ is given by
\bd
A_{\pi_1} = \left[ \ba{c} (A^2)^1 \\ (A^2)^2 \\ (A^1)^3 \\ (A^2)^4 
\ea \right] .
\ed
Thus the first, third, and fourth rows of $A_{\pi_1}$ come from $A^2$,
while the second row comes from $A^1$.

Next, suppose $\pi_2$ is a probabilistic policy, represented by
the matrix
\bd
M_2 = \left[ \ba{cc} 0.3 & 0.7 \\ 0.2 & 0.8 \\ 0.9 & 0.1 \\ 0.4 & 0.6 
\ea \right] .
\ed
Thus, if $X_t = x_1$, then the action $U_t = u_1$ with probability $0.3$
and equals $u_2$ with probability $0.7$, and so on.
For this policy, the resulting state transition matrix is determined
as follows:
\bd
A_{\pi_2} = \left[ \ba{c} 0.3 (A^1)^1 + 0.7 (A^2)^1 \\
0.2 (A^1)^2 + 0.8 (A^2)^2 \\
0.9 (A^1)^3 + 0.1 (A^2)^3 \\
0.4 (A^1)^4 + 0.6 (A^2)^4 
\ea \right] .
\ed
\halmos
\end{example}

For a MDP, one can pose three questions:
\ben
\item \textbf{Policy evaluation:}
We have seen already that, given a Markov reward process, with
a reward vector $\rbold$ and a discount factor $\g$, 
there corresponds a unique value vector $\v$.
We have also seen that, for any choice of a policy $\pi$, whether
deterministic or probabilistic,
there corresponds a state transition matrix $A_\pi$ and
a reward vector $\rbold_\pi$.
Therefore, once a policy $\pi$ is chosen, the Markov \textit{decision}
process becomes a Markov
\textit{reward} process with state transition matrix $A_\pi$ and reward vector
$\rbold_\pi$.
We can define $\v_\pi(x_i)$ to be
the value vector associated with this Markov reward process.
The question is: How can $\v_\pi(x_i)$ be computed?
\item \textbf{Optimal Value Determination:}
For each policy $\pi$, there is an associated value vector $\v_\pi$.
Let us view $\v_\pi$ as a map from $\X$ to $\R$, so that $V_\pi(x_i)$
is the $i$-th component of $\v_\pi$.
Now suppose $x_i \in \X$ is a specified initial state, and define
\be\label{eq:S112}
V^*(x_i) := \max_{\pi \in \Pi_p} V_\pi(x_i) ,
\ee
to be the \textbf{optimal value} over all policies, when the MDP
is started in the initial state $X_0 = x_i$.
How can $V^*(x_i)$ be computed?
Note that in \eqref{eq:S112}, the optimum is taken over all
\textit{probabilistic} policies.
However, it can be shown that the optimum is the same even if
$\pi$ is restricted to only \textit{deterministic} policies.
\item \textbf{Optimal Policy Determination:}
In \eqref{eq:S112} above, we associate an optimal policy with
each state $x_i$.
Now we can extend the idea and
define the \textbf{optimal policy} map $\X \ap \Pi_d$ via
\be\label{eq:S113}
\pi^*(x_i) := \argmax_{\pi \in \Pi_d} V_\pi(x_i) .
\ee
How can the optimal policy map $\pi^*$ be determined?
Note that it is not \textit{a priori} evident that there exists
\textit{one} policy that is optimal for \textit{all} initial states.
But the existence of such an optimal policy can be shown.
Also, we can restrict to $\pi \in \Pi_d$ in \eqref{eq:S113}
because it can be shown that
the maximum over $\pi \in \Pi_p$ is not any larger.
In other words,
\bd
\max_{\pi \in \Pi_d} V_\pi(x_i) = \max_{\pi \in \Pi_p} V_\pi(x_i) .
\ed
\een

\subsection{Markov Decision Processes: Solution}\label{ssec:32}

In this subsection we present answers to the three questions above.

\subsubsection{Policy Evaluation:}

Suppose a policy $\pi$ in $\Pi_d$ or $\Pi_p$ is specified.
Then the corresponding state transition matrix $A_\pi$ and reward 
vector $\rbold_\pi$ are given by
\eqref{eq:S17} (or \eqref{eq:S19}) and \eqref{eq:S110} respectively.
As pointed out above, once the policy is chosen, the process
becomes just a Markov reward process.
Then it readily follows from Theorem \ref{thm:S11} that
$\v_\pi$ satisfies an equation analogous to \eqref{eq:S15}, namely
\be\label{eq:S116}
\v_\pi = \rbold_\pi + \g A_\pi \v_\pi .
\ee
As before, it is inadvisable to compute $\v_\pi$ via
$\v_\pi = (I -\g A_\pi)^{-1} \rbold_\pi$.
Instead, one should use value iteration to solve \eqref{eq:S116}.
Observe that, whatever the policy $\pi$ might be, the resulting
state transition matrix $A_\pi$ satisfies $\nmm{A}_{\infty \ap \infty} = 1$.
Therefore the map $\y \mapsto \rbold_\pi + \g A_\pi \y$ is a contraction
with respect to $\nmi{\cdot}$, with contraction constant $\g$.

\subsubsection{Optimal Value Determination:}


Now we introduce one of the key ideas in Markov Decision Processes.
Define the \textbf{Bellman iteration map} $B: \R^n \ap \R^n$ via
\be\label{eq:S116f}
(B\v)_i := \max_{u_k \in \U} \left[ R(x_i,u_k)
+ \g \sum_{j=1}^n a_{ij}^{u_k} v_j \right] .
\ee

\begin{theorem}\label{thm:S17}
The map $B$ is monotone and a contraction with respect to the
$\ell_\infty$-norm.
Therefore the fixed point $\vbb$ of the map $B$ satisfies the relation
\be\label{eq:S116g}
(\vbb)_i := \max_{u_k \in \U} \left[ R(x_i,u_k)
+ \g \sum_{j=1}^n a_{ij}^{u_k} (\vbb)_j \right] .
\ee
\end{theorem}
Note that \eqref{eq:S116g} is known as the \textbf{Bellman Optimality equation}.
Thus, in principle at least, we can choose an arbitrary initial guess
$\v_0 \in \R^d$, and repeatedly apply the Bellman iteration.
The resulting iterations would converge to the unique fixed point
of the operator $B$, which we denote by $\vbb$.

The significance of the Bellman iteration is given by the next theorem.

\begin{theorem}\label{thm:S17a}
Define $\vbb \in \R^n$ to be the unique fixed point of $B$, and define
$\v^* \in \R^n$ to equal $[V^*(x_i), x_i \in \X]$, where
$V^*(x_i)$ is defined in \eqref{eq:S112}.
Then $\vbb = \v^*$.
\end{theorem}

Therefore, the optimal value vector can be computed using the
Bellman iteration.
However, knowing the optimal \textit{value} vector does not, by itself,
give us an optimal \textit{policy}.

\subsubsection{Optimal Policy Determination}

To solve the problem of optimal policy determination,
we introduce another function $Q_\pi: \X \times \U \ap \R$,
known as the \textbf{action-value function}, which is defined as follows:
\be\label{eq:S16a}
Q_\pi(x_i,u_k) := R(x_i,u_k) + E_\pi \left[ \sum_{t=1}^\infty \g^t R_\pi(X_t) | 
X_0 = x_i , U_0 = u_k \right] .
\ee
This function was first defined in \cite{Watkins-Dayan92}.
Note that $Q_\pi$ is defined only for deterministic policies.
In principle it is possible to define it for probabilistic policies,
but this is not commonly done.
In the above definition, the expectation $E_\pi$ is with respect to
the evolution of the state $X_t$ under the policy $\pi$.

The way in which a MDP is set up is that at time $t$, the Markov process
reaches a state $X_t$, based on the previous state $X_{t-1}$ and the
state transition matrix $A_\pi$ corresponding to the policy $\pi$.
Once $X_t$ is known, the policy $\pi$ determines the action $U_t = \pi(X_t)$,
and then the reward $R_\pi(X_t) = R(X_t,\pi(X_t))$ is generated.
In particular, when defining the value function $V_\pi(x_i)$ corresponding
to a policy $\pi$, we start off the MDP in the initial state $X_0 = x_i$,
\textit{and choose the action $U_0 = \pi(x_i)$}.
However, in defining the action-value function $Q_\pi$, we do not feel compelled
to set $U_0 = \pi(X_0) = \pi(x_i)$, and can
choose an \textit{arbitrary action} $u_k \in \U$.
From $t=1$ onwards however, the action $U_t$ is chosen as $U_t = \pi(X_t)$.
This seemingly small change leads to some simplifications.

Just as we can interpret $V_\pi:\X \ap \R$ as an $n$-dimensional vector,
we can interpret $Q_\pi: \X \times \U \ap \R$ as an $nm$-dimensional
vector, or as a matrix of dimension $n \times m$.
Consequently the $Q_\pi$-vector has higher dimension than the value vector.

\begin{theorem}\label{thm:S12}
For each policy $\pi \in \Pi_d$,
the function $Q_\pi$ satisfies the recursive relationship
\be\label{eq:S16b}
Q_\pi(x_i,u_k) = R(x_i,u_k) +
\g \sum_{j=1}^n a_{ij}^{u_k} Q_\pi(x_j,\pi(x_j)) .
\ee
\end{theorem}

\begin{proof}
Observe that at time $t=0$, the state transition matrix is $A^{u_k}$.
So, given that $X_0 = x_i$ and $U_0 = u_k$, the next state $X_1$ has
the distribution 
\bd
X_1 \sim [ a_{ij}^{u_k} , j = 1 , \cdots , n ] .
\ed
Moreover, $U_1 = \pi(X_1)$ because the policy $\pi$ is implemented
from time $t=1$ onwards.
Therefore
\begin{eqnarray*}
Q_\pi(x_i,u_k) & = & R(x_i,u_k) \\
& + & E_\pi \left[ \sum_{j=1}^n a_{ij}^{u_k}
\left( \g R(x_j,\pi(x_j)) + \sum_{t=2}^\infty \g^t R_\pi(X_t) | X_1 = x_j ,
U_1 = \pi(x_j) \right) \right] \\
& = & R(x_i,u_k) \\
& + & E_\pi \left[ \g \sum_{j=1}^n a_{ij}^{u_k}
\left( R(x_j,\pi(x_j)) + \sum_{t=1}^\infty \g^t R_\pi(X_t) | X_1 = x_j ,
U_1 = \pi(x_j) \right) \right] \\
& = & R(x_i,u_k) + \g \sum_{j=1}^n a_{ij}^{u_k} Q(x_j,\pi(x_j)) .
\end{eqnarray*}
This is the desired conclusion.
\end{proof}

\begin{theorem}\label{thm:S13}
The functions $V_\pi$ and $Q_\pi$ are related via
\be\label{eq:S16c}
V_\pi(x_i) = Q_\pi(x_i,\pi(x_i)) .
\ee
\end{theorem}

\begin{proof}
If we choose $u_k = \pi(x_i)$ then \eqref{eq:S16b} becomes
\bd
Q_\pi(x_i,\pi(x_i)) = R_\pi(x_i)
+ \g \sum_{j=1}^n a_{ij}^{\pi(x_j)} Q(x_j,\pi(x_j)).
\ed
This is the same as \eqref{eq:S16} written out componentwise.
We know that \eqref{eq:S16} has a unique solution, namely $V_\pi$.
This shows that \eqref{eq:S16c} holds.
\end{proof}

The import of Theorem \ref{thm:S13} is the following:
In defining the function $Q_\pi(x_i,u_k)$ for a fixed policy $\pi \in \Pi_d$,
we have the freedom to choose the initial action $u_k$ as any
element we wish in the action space $\U$.
However, if we choose the initial action $u_k = \pi(x_i)$ for each
state $x_i \in \X$, then the corresponding action-value function 
$Q_\pi(x_i,u_k)$  equals the value function $V_\pi(x_i)$, for each
state $x_i \in \X$.

In view of \eqref{eq:S16c}, the recursive equation for $Q_\pi$ can be
rewritten as
\be\label{eq:S116d}
Q_\pi(x_i,u_k) = R(x_i,u_k) + \g \sum_{j=1}^n a_{ij}^{u_k} V_\pi(x_j) .
\ee
This motivates the next theorem.

\begin{theorem}\label{thm:S15}
Define $Q^*: \X \times \U \ap \R$ by
\be\label{eq:S117a}
Q^*(x_i,u_k) = R(x_i,u_k) + \g \sum_{j=1}^n a_{ij}^{u_k} V^*(x_j) .
\ee
Then $Q^*(\cdot,\cdot)$ satisfies the following relationships:
\be\label{eq:S117b}
Q^*(x_i,u_k) = R(x_i,u_k) + \g \sum_{j=1}^n a_{ij}^{u_k}
\max_{w_l \in \U} Q^*(x_j,w_l) .
\ee
\be\label{eq:S117c}
V^*(x_i) = \max_{u_k \in \U} Q^*(x_i,u_k) , 
\ee
Moreover, every policy $\pi \in \Pi_d$ such that
\be\label{eq:S117d}
\pi^*(x_i) = \argmax_{u_k \in \U} Q^*(x_i,u_k) 
\ee
is optimal.
\end{theorem}

\begin{proof}
Since $Q^*(\cdot,\cdot)$ is defined by \eqref{eq:S117a}, it follows that
\bd
\max_{u_k \in \U} Q^*(x_i,u_k) =
\max_{u_k \in \U} \left[ R(x_i,u_k) + \g \sum_{j=1}^n a_{ij}^{u_k} V^*(x_j)
\right] = V^*(x_i) ,
\ed
This establishes \eqref{eq:S117c} and \eqref{eq:S117d}.
Substituting from \eqref{eq:S117c} into \eqref{eq:S117a} gives
\eqref{eq:S117b}.
\end{proof}

Theorem \ref{thm:S15} converts the problem of determining an optimal
policy into one of solving the implicit equation \eqref{eq:S117b}.
For this purpose, we define an iteration on action-functions that is analogous
to \eqref{eq:S116f} for value functions.
As with the value function, the action-value function can either
be viewed as a map $Q: \X \times \U \ap \R$, or as a vector in $\R^{nm}$,
or as an $n \times m$ matrix.
We use whichever interpretation is convenient in the given situation.

\begin{theorem}\label{thm:S18}
Define $F: \R^{|\X| \times |\U|} \ap \R^{|\X| \times |\U|}$ by
\be\label{eq:S120a}
[F(Q)](x_i,u_k) := R(x_i,u_k) + \g \sum_{j=1}^n a_{ij}^{u_k}
\max_{w_l \in \U} Q(x_j,w_l) .
\ee
Then the map $F$ is monotone and is a contraction.
Moreover, for all $Q_0: \X \times \U \ap \R$, the sequence of iterations
$\{ F^t(Q_0) \}$ converges to $Q^*$ as $\tai$.
\end{theorem}

If we were to rewrite \eqref{eq:S116g} and \eqref{eq:S117b} in terms of
expected values, the differences between the $Q$-function 
and the $V$-function would become apparent.
We can rewrite \eqref{eq:S116g} as
\be\label{eq:S120d}
V^*(X_t) = \max_{U_t \in \U} \{ R(X_t,U_t) + 
\g E[ V^*(X_{t+1}) | X_t ] \} ,
\ee
and \eqref{eq:S117b} as
\be\label{eq:S120e}
Q^*(X_t,U_t) = R(X_t,U_t) + \g E \left[ \max_{U_{t+1} \in \U}
Q^*(X_{t+1},U_{t+1}) \right] .
\ee
Thus in the Bellman formulation and iteration, the maximization occurs
\textit{outside} the expectation, whereas with the $Q$-formulation and
$F$-iteration, the maximization occurs \textit{inside} the expectation.

\subsection{Iterative Algorithms for MDPs with Unknown Dynamics}\label{ssec:33}

In principle, Theorem \ref{thm:S11a} can be used to compute,
to arbitrary precision, the value vector of a Markov reward process.
Similarly, Theorem \ref{thm:S18} can be use to compute, to arbitrary precision,
the optimal action-value function of a Markov Decision Process,
from which both the optimal value function and the optimal policy can
be determined.
However, both theorems depend crucially on \textit{knowing the dynamics
of the underlying process}.
For instance, if the state transition matrix $A$ is not known, it would
not be possible to carry out the iterations
\bd
\y^{i+1} = \rbold + \g A \y^i .
\ed

Early researchers in Reinforcement Learning were aware of this issue,
and developed several algorithms that do not require \textit{explicit}
knowledge of the dynamics of the underlying process.
Instead, it is assumed that a sample path $\{ X_t \}_{t=0}^\infty$
of the Markov process, together with the associated reward process,
are available for use.
With this information, one can think of two distinct approaches.
First, one can use the sample path to \textit{estimate} the state
transition matrix, call it $\Ah$.
After a sufficiently long sample path has been observed,
the contraction iteration above can be applied with $A$ replaced by $\Ah$.
This would correspond to so-called ``indirect adaptive control.''
The second approach would be to use the sample path right from time $t = 0$,
and adjust \textit{only one} component of the estimated value function
at each time instant $t$.
This would correspond to so-called ``direct adaptive control.''
Using a similar approach, it is also possible to estimate the action-value
function based on a single sample path.
We describe two such algorithms, namely Temporal Difference learning
for estimating the value function of a Markov reward process,
and $Q$-learning for estimating the action-value function of a Markov
Decision Process.
Within Temporal Difference, we make a further distinction between
estimating the full value vector, and estimating a projection of
the value vector onto a lower-dimensional subspace.

\subsubsection{Temporal Difference Learning Without Function Approximation}

In this subsection and the next,
we describe the so-called ``temporal difference''
family of algorithms, first introduced in \cite{Sutton88}.
The objective of the algorithm is to compute the value vector of
a Markov reward process.
Recall that the value vector $\v$ of a Markov reward process
satisfies \eqref{eq:S15}.
In Temporal Difference approach, it is \textit{not} assumed that the state
transition matrix $A$ is known.
Rather, it is assumed that the learner has available a
sample path $\{ (X_t \}$ of the Markov process under study,
together with the associated reward at each time.
For simplicity it is assumed that the reward is deterministic and not random.
Thus the reward at time $t$ is just $R(X_t)$ and does not add any information.

There are two variants of the algorithm.
In the first, one constructs a sequence of approximations $\vbh_t$
that, one hopes, would converge to the true value vector $\v$ as $\tai$.
In the second, which is used when $n$ is very large, one chooses
a ``basis representation matrix'' $\Psi \in \R^{n \times d}$, where
$d \ll n$.
Then one constructs a sequence of vectors $\bth_t \in \R^d$,
such that the corresponding sequence of vectors $\Psi \bth_t \in \R^n$
forms an approximation to the value vector $\v$.
Since there is no \textit{a priori} reason to believe that $\v$
belongs to the range of $\Psi$, there is also no reason to believe
that $\Psi \bth_t$ would converge to $\v$.
The second approach is called Temporal Difference Learning with
function approximation.
The first is studied in this subsection, while the second is
studied in the next subsection.

In principle, by observing the sample path for a sufficiently long
duration, it is possible to make a reliable estimate of $A$.
However, a key feature of the temporal difference algorithm is that
it is a ``direct'' method, which works directly with the sample path,
without attempting to infer the underlying Markov process.
With the sample path $\{ X_t \}$ of the Markov process,
one can associate a corresponding ``index process''
$\{ N_t \}$ taking values in $[n]$, as follows:
\bd
N_t = i \mbox{ if } X_t = x_i \in \X .
\ed
It is obvious that the index process has the same transition matrix
$A$ as the process $\{ X_t \}$.
The idea is to start with an initial estimate $\vbh_0$, and update it
at each time $t$ based on the sample path $\{ (X_t,R(X_t)) \}$.

Now we introduce the $\TDl$ algorithm studied in this paper.
This version of the $\TDl$ algorithm comes from
\cite[Eq.\ (4.7)]{Jaakkola-et-al94}, and is as follows:
Let $\v^*$ denote the unique solution of the equation\footnote{For
clarity, we have changed the notation so that the value vector is now
denoted by $\v^*$ instead of $\v$ as in \eqref{eq:S15}.}
\bd
\v^* = \rbold + \g A \v^* .
\ed
At time $t$, let $\vbh_t \in \R^n$ denote the current estimate of $\v^*$.
Thus the $i$-th component of $\vbh_t$, denoted by $\Vh_{t,i}$,
is the estimate of the reward when the initial state is $x_i$.
Let $\{ N_t \}$ be the index process defined above.
Define the ``temporal difference''
\be\label{eq:324}
\d_{t+1} := R_{N_t} + \g \Vh_{t,N_{t+1}} - \Vh_{t,N_t} , \fa t \geq 0 ,
\ee
where $\Vh_{t,N_t}$ denotes the $N_t$-th component of the vector $\vbh_t$.
Equivalently, if the state at time $t$ is $x_i \in \X$ and the
state at the next time $t+1$ is $x_j$, then
\be\label{eq:322a}
\d_{t+1} = R_i + \g \Vh_{t,j} - \Vh_{t,i} .
\ee
Next, choose a number $\l \in [0,1)$.
Define the ``eligibility vector''
\be\label{eq:323}
\z_t = \sum_{\t = 0}^t (\g \l)^\t I_{ \{ N_{t-\t} = N_t \} } \eb_{N_{t-\t}} ,
\ee
where $\eb_{N_s}$ is a unit vector with a $1$ in location $N_s$
and zeros elsewhere.
Since the indicator function in the above summation picks up only
those occurrences where $N_{t-\t} = N_t$, the vector
$\z_t$ can also be expressed as
\be\label{eq:323a}
\z_t = z_t \eb_{N_t} ,
z_t = \sum_{\t = 0}^t (\g \l)^\t I_{ \{ N_{t-\t} = N_t \} } .
\ee
Thus the support of the vector $\z_t$ consists of the singleton $\{ N_t \}$.
Finally, update the estimate $\vbh_t$ as
\be\label{eq:325}
\vbh_{t+1} = \vbh_t + \d_{t+1} \al_t \z_t ,
\ee
where $\{ \al_t \}$ is a sequence of step sizes.
Note that, at time $t$, only the $N_t$-th component of $\vbh_t$ is
updated, and the rest remain the same.

A sufficient condition for the convergence of the $\TDl$-algorithm
is given in \cite{Jaakkola-et-al94}.

\begin{theorem}\label{thm:TDl}
The sequence $\{ \vbh_t \}$ converges almost surely to $\v^*$
as $\tai$, provided
\be\label{eq:325a}
\sum_{t=0}^\infty \al_t^2 I_{ \{ N_t = i \} } < \infty , \as ,
\fa i \in [n] ,
\ee
\be\label{eq:325b}
\sum_{t=0}^\infty \al_t I_{ \{ N_t = i \} } = \infty , \as ,
\fa i \in [n] ,
\ee
\end{theorem}

\subsubsection{$TD$-Learning with Function Approximation}\label{ssec:52}

In this set-up, we again observe a time series $\{ (X_t, R(X_t)) \}$.
The new feature is that there is a ``basis'' matrix $\Psi \in \R^{n \times d}$,
where $d \ll n$.
The estimated value vector at time $t$ is given by $\vh_t = \Psi \bth_t$,
where $\bth_t \in \R^d$ is the parameter to be updated.
In this representation, it is clear that, for any index $i \in [n]$, we have
that 
\bd
\Vh_{t,i} = \Psi^i \bth_t = \IP{(\Psi^i)^\top}{\bth_t} ,
\ed
where $\Psi^i$ denotes the $i$-th row of the matrix $\Psi$.

Now we define the learning rule for updating $\bth_t$.
Let $\{ X_t \}$ be the observed sample path.
By a slight abuse of notation, define
\bd
\y_t = [\Psi^{X_t} ]^\top \in \R^d .
\ed
Thus, if $X_t = x_i$, then $\y_t = [\Psi^i]^\top$.
The \textbf{eligibility vector} $\z_t \in \R^d$ is defined via
\be\label{eq:521}
\z_t = \sum_{\t=0}^t (\g \l)^{t-\t} \y_\t .
\ee
Note that $\z_t$ satisfies the recursion
\bd
\z_t = \g \l \z_{t-1} + \y_t .
\ed
Hence it is not necessary to keep track of an ever-growing set of past
values of $\y_\t$.
In contrast to \eqref{eq:323}, there is no term of the type
$I_{ \{ N_{t-\t} = N_t \} }$ in \eqref{eq:521}.
Thus, unlike the eligibility vector defined in \eqref{eq:323},
the current vector $\z_t$ can have more than one nonzero component.
Next, define the temporal difference $\d_{t+1}$ as in \eqref{eq:324}.
Note that, if $X_t = x_i$ and $X_{t+1} = x_j$, then
\bd
\d_{t+1} = r_i + \g [\Psi^j]^\top \bth_t - [\Psi^i]^\top \bth_t .
\ed
Then the updating rule is
\be\label{eq:523}
\bth_{t+1} = \bth_t + \al_t \d_{t+1} \z_t ,
\ee
where $\al_t$ is the step size.

The convergence analysis of \eqref{eq:523} is carried out in detail in
\cite{Tsi-Van-TAC97}, based on the assumption that the state transition
matrix $A$ is irreducible.
This is quite reasonable, as it ensures that every state $x_i$ occurs
infinitely often in any sample path, with probability one.
Since that convergence analysis does not readily fit into the methods
studied in subsequent sections, we state the main results without proof.
However, we state and prove various intermediate results, that are
useful in their own right.

%
Suppose $A$ is row-stochastic and irreducible, and let $\bmu$ denote
its stationary distribution.
Define $M = {\rm Diag}(\mu_i)$ and define a norm $\nmm{\cdot}_M$ on $\R^d$ by
\bd
\nmm{\v}_M = ( \v^\top M \v)^{1/2} .
\ed
Then the corresponding distance between two vectors $\v_1,\v_2$ is given by
\bd
\nmm{\v_1 - \v_2}_M = ( (\v_1 - \v_2)^\top M (\v_1 - \v_2) )^{1/2} .
\ed
Then the following result is proved in \cite{Tsi-Van-TAC97}.

\begin{lemma}\label{lemma:51}
Suppose $A \in [0,1]^{n \times n}$, is row-stochastic, and irreducible.
Let $\bmu$ be the stationary distribution of $A$.
Then
\bd
\nmm{A\v}_M \leq \nmm{\v}_M , \fa \v \in \R^n .
\ed
Consequently, the map $\v \mapsto \rbold + \g A \v$ is a contraction with
respect to $\nmm{\cdot}_M$.
\end{lemma}

\begin{proof}
We will show that
\bd
\nmm{A\v}_M^2 \leq \nmm{\v}_M^2 , \fa \v \in \R^n ,
\ed
which is clearly equivalent to the $\nmm{A\v}_M \leq \nmm{\v}_M$.
Now
\bd
\nmm{A\v}_M^2 =\sum_{i=1}^n \mu_i (A\v)_i^2
= \sum_{i=1}^n \mu_i \left( \sum_{j=1}^n A_{ij} v_j \right)^2 .
\ed
However, for each fixed index $i$, the row $A^i$ is a probability
distribution, and the function $f(Y) = Y^2$ is convex.
If we apply Jensen's inequality with $f(Y) = Y^2$, we see that
\bd
\left( \sum_{j=1}^n A_{ij} v_j \right)^2 \leq 
\sum_{j=1}^n A_{ij} v_j^2 , \fa i .
\ed
Therefore
\begin{eqnarray*}
\nmm{A\v}_M^2 & \leq & \sum_{i=1}^n \mu_i \left( \sum_{j=1}^n A_{ij} v_j^2 \right)
= \sum_{j=1}^n \left( \sum_{i=1}^n \mu_i A_{ij} \right) v_j^2 \\
& = & \sum_{j=1}^n \mu_j v_j^2 = \nmm{\v}_M^2 ,
\end{eqnarray*}
where in the last step we use the fact that $\bmu A = \bmu$.
\end{proof}

To analyze the behavior of the $\TDl$ algorithm with function approximation,
the following map $\Tl: \R^n \ap \R^n$ is defined in \cite{Tsi-Van-TAC97}:
\bd
[ \Tl\v]_i := (1-\l) \sum_{l=0}^\infty \l^l
E \left[ \sum_{\t=0}^l \g^\t R(X_{\t+1})
+ \g^{l+1} V_{X_{l+1}} | X_0 \ x_i \right] .
\ed
Note that $\Tl \v$ can be written explicitly as
\bd
\Tl \v = (1-\l) \sum_{l=0}^\infty \l^l
\left[ \sum_{\t=0}^l \g^\t A^\t \rbold + \g^{l+1} A^{l+1} \v \right] .
\ed

\begin{lemma}\label{lemma:52}
The map $\Tl$ is a contraction with respect to $\nmm{\cdot}_M$, with
contraction constant $[\g(1-\l)]/(1 - \g \l )$.
\end{lemma}

\begin{proof}
Note that the first term on the right side of does not
depend on $\v$.
Therefore
\bd
\Tl (\v_1 - \v_2) = \g (1 - \l )
\sum_{l = 0}^\infty (\g \l)^l A^{l+1} ( \v_1 - \v_2 ).
\ed
However, it is already known that
\bd
\nmm{ A(\v_1 - \v_2) }_M \leq \nmm{\v_1 - \v_2}_M .
\ed
By repeatedly applying the above, it follows that
\bd
\nmm{ A^l(\v_1 - \v_2) }_M \leq \nmm{\v_1 - \v_2}_M , \fa l .
\ed
Therefore
\bd
\nmm{\Tl ( \v_1 - \v_2 ) }_M \leq
\g ( 1 - \l ) \sum_{l=0}^\infty (\g \l)^l \nmm{\v_1 - \v_2}_M
= \frac{ \g(1-\l)}{1-\g \l} \nmm{\v_1 - \v_2}_M .
\ed
This is the desired bound.
\end{proof}

Define a projection $\Pi: \R^n \ap \R^n$ by
\bd
\Pi \a := \Psi ( \Psi^\top M \Psi)^{-1} \Psi^\top M \a .
\ed
Then
\bd
\Pi \a = \argmin_{\b \in \Psi(\R^d)} \nmm{\a-\b}_M .
\ed
Thus $\Pi$ projects the space $\R^n$ onto the image of the matrix $\Psi$,
which is a $d$-dimensional subspace, if $\Psi$ has full column rank.
In other words,
$\Pi \a$ is the closest point to $\a$ in the subspace $\Psi(\R^n)$.

Next, observe that the projection $\Pi$ is nonexpansive with respect
to $\nmm{\cdot}_M$.
As a result, the composite map $\Pi \Tl$ is a contraction.
Thus there exists a unique $\vb \in \R^d$ such that
\bd
\Pi \Tl \vb = \vb .
\ed
Moreover, the above equation shows that in fact $\vb$ belongs to
the range of $\Psi$.
Thus there exists a $\bths \in \R^d$ such that $\vb = \Psi \bths$,
and $\bths$ is unique if $\Psi$ has full column rank.

The limit behavior of the $\TDl$ algorithm is given by the next theorem,
which is a key result from \cite{Tsi-Van-TAC97}.

\begin{theorem}\label{thm:52}
Suppose that $\Psi$ has full column rank, and that
\bd
\sum_{t=0}^\infty \al_t = \infty , \sum_{t=0}^\infty \al_t^2 < \infty .
\ed
Then the sequence $\{ \bth_t \}$ converges almost surely to $\bth^* \in \R^d$,
where $\bth^*$ is the unique solution of
\bd
\Pi \Tl ( \Psi \bth^* ) = \Psi \bth^* .
\ed
Moreover
\bd
\nmm{ \Psi \bth^* - \v^*}_M \leq \frac{ 1 - \g \l}{1 - \g}
\nmm{ \Pi \v^* - \v^*}_M .
\ed
\end{theorem}


Note that, since $\Psi \bth \in \Pi(\R^d)$ for all $\bth \in \R^d$,
the best that one can hope for is that
\bd
\nmm{ \Psi \bths - \v^*}_M = \nmm{ \Pi \v^* - \v^*}_M .
\ed
The theorem states that the above identity might not hold, and provides
an upper bound for the
distance between the limit $\Psi \bth^*$ and the true value vector $\v^*$.
It is bounded by a factor $(1-\g \l)/(1 - \g)$ times this minimum.

Note that $(1-\g \l)/(1 - \g) > 1$.
So this is the extent to which the $\TDl$ iterations miss the optimal
approximation.

\subsubsection{$Q$-Learning}

The $Q$-learning algorithm proposed in \cite{Watkins-Dayan92} has the
characterization \eqref{eq:S117b} of $Q^*$ as its starting point.
The algorithm is based on the following premise:
At time $t$, the current state $X_t$ can be observed; call it $x_i \in \X$.
Then the learner is free to choose the action $U_t$; call it $u_k \in \U$.
With this choice, the next state $X_{t+1}$ has the probability
distribution equal to the $i$-th row of the state transition matrix $A^{u_k}$.
Suppose the observed next stat $X_{t+1}$ is $x_j \in \X$.
With these conventions, the $Q$-learning algorithm proceeds as follows.

\ben
\item Choose an arbitrary initial guess $Q_0 : \X \times \U \ap \R$
and an initial state $X_0 \in \X$.
\item At time $t$, with current state $X_t = x_i$, choose a current action
$U_t = u_k \in \U$, and let the Markov process run for one time step.
Observe the resulting next state $X_{t+1} = x_j$.
Then update the function $Q_t$ as follows:
\be\label{eq:3311}
\begin{split}
Q_{t+1}(x_i,u_k) & =
Q_t(x_i,u_k) + \al_t [ R(x_i,u_k) + \g V_t(x_j) - Q_t(x_i,u_k) ] , \\
Q_{t+1}(x_s,w_l) & = Q_t(x_s,w_l) , \fa (x_s,w_l) \neq (x_i,u_k) .
\end{split}
\ee
where
\be\label{eq:3312}
V_t(x_j) = \max_{w_l \in \U} Q_t(x_j,w_l) ,
\ee
and $\{ \al_t \}$ is a deterministic sequence of step sizes.
\item Repeat.
\een

It is evident that in the $Q$-learning algorithm, at any instant of time
$t$, only one element (namely $Q(X_t,U_t)$) gets updated.
In the original paper by Watkins and Dayan \cite{Watkins-Dayan92},
the convergence of the algorithm used some rather \textit{ad hoc} methods.
Subsequently, a general class of algorithms known as ``asynchronous
stochastic approximation,'' which included $Q$-learning as a special case,
was introduced in \cite{Tsi-ML94,Jaakkola-et-al94}. 
A sufficient condition for the convergence of the $Q$-learning algorithm,
which was originally presented in \cite{Watkins-Dayan92},
is rederived using these methods.
\begin{theorem}\label{thm:QL}
The $Q$-learning algorithm converges to the optimal action-value
function $Q^*$ provided the following conditions are satisfied.
\be\label{eq:3313}
\sum_{t=0}^\infty \al_t I_{(X_t,U_t) = (x_i,u_k)} = \infty,
\fa (x_i,u_k) \in \X \times \U ,
\ee
\be\label{eq:3314}
\sum_{t=0}^\infty \al_t^2 I_{(X_t,U_t) = (x_i,u_k)} < \infty,
\fa (x_i,u_k) \in \X \times \U .
\ee
\end{theorem}

The main shortcoming of Theorems \ref{thm:TDl} and \ref{thm:QL}
is that the sufficient conditions \eqref{eq:325a}, \eqref{eq:325b},
\eqref{eq:3313} and \eqref{eq:3314}
are \textit{probabilistic} in nature.
Thus it is not clear how they are to be verified in a specific application.
Note that in the $Q$-learning algorithm, there is no guidance on how
to choose the next action $U_t$.
Presumably $U_t$ is chosen so as to ensure that \eqref{eq:3313} and
\eqref{eq:3314} are satisfied.
In Section \ref{sec:Appl}, we show how these theorems can be proven,
and also, how the troublesome probabilistic sufficient conditions
can be replaced by purely algebraic conditions.

\section{Stochastic Approximation Algorithms}\label{sec:SA}

\subsection{Stochastic Approximation and Relevance to RL}\label{ssec:41}

The contents of the previous section make it clear that in MDP theory,
a central role is played by \textit{the need to solve fixed-point problems}.
Determining the value of a Markov reward problem requires the solution of
\eqref{eq:S15}.
Determining the optimal value of an MDP requires finding the fixed point
of the Bellman iteration.
Finally, determining the optimal policy for an MDP requires finding
the fixed point of the $F$-iteration.
As pointed out in Section \ref{ssec:33}, when the dynamics of an MDP
are completely known, these fixed point problems can be solved by
repeatedly applying the corresponding contraction mapping.
However, when the dynamics of the MDP are not known, and one has access
only to a sample path of the MDP, a different approach is required.
In Section \ref{ssec:33}, we have presented two such methods, namely
the Temporal Difference algorithm for value determination, and the
$Q$-Learning algorithm for determining the optimal action-value function.
Theorems \ref{thm:TDl} and \ref{thm:QL} respectively give sufficient
conditions for the convergence of these algorithms.
The proofs of these theorems, as given in the original papers,
tend to be ``one-off,'' that is, tailored to the specific algorithm.
It is now shown that a probabilistic
method known as ``stochastic approximation '' (SA) can be used to unify 
these methods in a common format.
Moreover, instead of the convergence proofs being ``one-off,''
the SA algorithm provides a unifying approach.

The applications of SA go beyond these two specific algorithms.
There is another area called ``Deep Reinforcement Learning'' for problems 
in which the size of the state space is very large.
Recall that the action-value function $Q: \X \times \U$ can either
be viewed as an $nm$-dimensional vector, or an $n \times m$ matrix.
In Deep RL, one determines (either exactly or approximately)
the action-value function $Q(x_i,u_k)$ for \textit{a small number}
of pairs $(x_i,u_k) \in \X \times \U$.
Using these as a starting point, the overall function $Q$ defined for
\textit{all} pairs $(x_i,u_k) \in \X \times \U$ is obtained by training
a deep neural network.
Training a neural network (in this or any other application) requires
the minimization of the average mean-squared error, denoted by $J(\bth)$
where $\bth$ denotes the vector of adjustable parameters.
In general, the function $J(\cdot)$ is not convex; hence one can at best
aspire to find a \textit{stationary point} of $J(\cdot)$, i.e.,
a solution to the equation $\gJ(\bth) = \bz$.
This problem is also amenable to the application of the SA approach.

Now we give a brief introduction to stochastic approximation.
Suppose $\f : \R^d \ap \R^d$ is some function, and $d$ can be any integer.
The objective of SA is to find a solution to the equation $\f(\bth) = \bz$,
when only noisy measurements of $\f(\cdot)$ are available.
The SA method was introduced in \cite{Robbins-Monro51},
where the objective was to find a solution to a \textit{scalar}
equation $f(\th) = 0$, where $f : \R \ap \R$.
The extension to the case where $d > 1$ was first proposed in
\cite{Blum54}.
The problem of finding a fixed point of a map $\gbold : \R^d \ap \R^d$,
can be formulated as the above problem with $\f(\bth) := \gbold(\bth) - \bth$.
If it is desired to find a stationary point of a $C^1$ function
$J : \R^d \ap \R$, then we simply set $\f(\bth) = \gJ(\bth)$.
Thus the above problem formulation is quite versatile.
More details are given at the start of Section \ref{ssec:42}.

Stochastic approximation is a family of  \textit{iterative} algorithms,
in which one begins with an initial guess $\bth_0$, and derives the
next guess $\bth_{t+1}$ from $\bth_t$.
Several variants of SA are possible.
In \textbf{synchronous SA}, \textit{every} component of $\bth_t$ is
changed to obtain $\bth_{t+1}$.
This was the original concept of SA.
If, at any time $t$, \textit{only one} component of $\bth_t$ is changed
to obtain $\bth_{t+1}$, and the others remain unchanged, this is
known as \textbf{asynchronous stochastic approximation (ASA)}.
This phrase was apparently first introduced in \cite{Tsi-ML94},
A variant of the approach in \cite{Tsi-ML94} is presented in \cite{Borkar98}.
Specifically, in \cite{Borkar98}, a distinction is introduced
between using a ``local clock'' versus using a ``global clock.''
It is also possible to study an intermediate situation where, at each time $t$,
\textit{some but not necessarily all} components of $\bth_t$
are updated.
There does not appear to be a common name for this situation.
The phrase \textbf{Batch Asynchronous Stochastic Approximation (BASA)}
is introduced in \cite{MV-BASA-arxiv22}.
More details about these variations are given below.
There is a fourth variant, known as \textbf{two time-scale SA} 
is introduced in \cite{Borkar97}.
In this set-up, one attempts to solve two \textit{coupled} equations
of the form
\bd
\f(\bth,\bphi) = \bz , \gbold(\bth,\bphi) = \bz ,
\ed
where $\bth \in \R^n , \bphi \in \R^m$, and
$\f : \R^n \times \R^m \ap \R^n , \gbold : \R^n \times \R^m \ap \R^m$.
the idea is that one of the iterations (say $\bth_{t+1}$) is updated
``more slowly'' than the other (say $\bphi_{t+1}$).
Due to space limitations, two time-scale SA is not discussed
further in this paper.
The interested reader is referred to \cite{Borkar97,Tadic-ACC04,CL-SB-Auto17}
for the theory, and to \cite{Konda-Borkar99,Konda-Tsi99} for applications
to a specific type of RL, known as \textbf{Actor-Critic Algorithms}.

The relevance of SA to RL arises from the following factors:
\bit
\item Many (though not all) algorithms used in RL can formulated as
some type of SA algorithms.
\item Examples include Temporal Difference Learning,
Temporal Difference Learning with function approximation, $Q$-Learning,
Deep Neural Network Learning, and Actor-Critic Learning.
The first three are discussed in detail in Section \ref{sec:Appl}.
\eit
Thus: SA provides a \textit{unifying framework} for several
disparate-looking RL algorithms.

\subsection{Problem Formulation}\label{ssec:42}

There are several equivalent formulations of the basic SA problem.
\ben
\item \textbf{Finding a zero of a function:}
Suppose $\f: \R^d \ap \R^d$ is some function.
Note that $\f(\cdot)$ need not be available in closed form.
The only thing needed is that, given any $\bth \in \R^d$, an ``oracle''
returns a noise-corrupted version of $\f(\bth)$.
The objective is to determine a solution of the equation $\f(\bth) = \bz$.
\item \textbf{Finding a fixed point of a mapping:}
Suppose $\gbold: \R^d \ap \R^d$.
The objective is to find a fixed point of $\gbold(\cdot)$, that is, a solution
to $\gbold(\bth) = \bth$.
If we define $\f(\bth) = \gbold(\bth) - \bth$, this is the same problem
as the above.
One might ask: Why not define $\f(\bth) = \bth - \gbold(\bth)$?
As we shall see below, the convergence of the SA algorithm (in various
forms) is closely related to the global asymptotic stability of the
ODE $\dot{\bth} = \f(\bth)$.
Also, as seen in the previous section, in many applications, the map
$\gbold(\cdot)$ of which we wish to find a fixed point is a contraction.
In such a case, there is a unique fixed point $\bths$ of $\gbold(\cdot)$.
In such a case, under relatively mild conditions $\bths$ is a globally
asymptotically stable equilibrium of the ODE
$\dot{\bth} = \gbold(\bth) - \bth$, but not if the sign is reversed.
\item \textbf{Finding a stationary point of a function:}
Suppose $J: \R^d \ap \R$ is a $\C^1$ function.
The objective is to find a stationary point of $J(\cdot)$, that is,
a $\bth$ such that $\gJ(\bth) = \bz$.
If we define $\f(\bth) = - \gJ(\bth)$, then this is the same problem
as above.
Here again, if we wish the SA algorithm to converge to a global \textit{minimum}
of $J(\cdot)$, then the minus sign is essential.
On the other hand, if we wish the SA algorithm to converge to a global
\textit{maximum} of $J(\cdot)$, then we remove the minus sign.
\een

Suppose the problem is one of finding a zero of a given function $\f(\cdot)$.
The \textbf{synchronous} version of SA proceeds as follows:
An initial guess $\bth_0 \in \R^d$ is chosen (usually in a deterministic
manner, but it can also be randomly chosen).
At time $t$, the available measurement is
\be\label{eq:411}
\y_{t+1} = \f(\bth_t) + \bxi_{t+1} ,
\ee
where $\bxi_{t+1}$ is the measurement noise.
Based on this, the current guess is updated to
\be\label{eq:412}
\bth_{t+1} = \bth_t + \al_t \y_{t+1} = \bth_t + \al_t [ \f(\bth_t) + \bxi_{t+1}] ,
\ee
where $\{ \al_t \}$ is a predefined sequence of ``step sizes,'' with
$\al_t \in (0,1)$ for all $t$.
If the problem is that of finding a fixed point of $\gbold(\cdot)$,
the updating rule is
\be\label{eq:413}
\bth_{t+1} = \bth_t + \al_t [ \gbold (\bth_t) -\bth_t + \bxi_{t+1}]
= (1 - \al_t) \bth_t + \al_t [ \gbold(\bth_t) + \bxi_{t+1} ] .
\ee
If the problem is to find a stationary point of $J(\cdot)$,
the updating rule is
\be\label{eq:414}
\bth_{t+1} = \bth_t + \al_t \y_{t+1} = \bth_t + \al_t [ -\gJ(\bth_t) + \bxi_{t+1}] .
\ee
These updating rules represent what might be
called \textbf{Synchronous SA}, because at each time $t$, \textit{every}
component of $\bth_t$ is updated.
Other variants of SA are studied in subsequent sections.

\subsection{A New Theorem for Global Asymptotic Stability}\label{ssec:43}

In this section we state a new theorem on the global asymptotic stability
of nonlinear ODEs.
This theorem is new and is of interest aside from its applications to
the convergence of SA algorithms.
The contents of this section and the next section
are taken from \cite{MV-GES-SA-arxiv22}.
To state the result (Theorem \ref{thm:41} below),
we introduce a few preliminary concepts from Lyapunov stability theory.
The required background can be found in \cite{MV-93,Hahn67,Khalil02}.

\begin{definition}\label{def:41}
A function $\phi : \R_+ \ap \R_+$ is said to \textbf{belong to class $\K$},
denoted by $\phi \in \K$, if $\phi(0) = 0$, and $\phi(\cdot)$ is
strictly increasing.
A function $\phi \in \K$ is said to \textbf{belong to class $\KR$},
denoted by $\phi \in \KR$, if in addition, $\phi(r) \ap \infty$
as $r \ap \infty$.
A function $\phi : \R_+ \ap \R_+$ is said to \textbf{belong to class $\B$},
denoted by $\phi \in \B$, if $\phi(0) = 0$, and in addition, for all
$0 < \e < M < \infty$ we have that
\be\label{eq:415}
\inf_{\e \leq r \leq M} \phi(r) > 0 .
\ee
\end{definition}

The concepts of functions of class $\K$ and class $\KR$ are standard.
The concept of a function of class $\B$ is new.
Note that, if $\phi(\cdot)$ is continuous, then it belongs to Class $\B$ if
and only if $\phi(0) = 0$, and $\phi(r) > 0$ for all $r > 0$.

\begin{example}\label{exam:41}
Observe that every $\phi$ of class $\K$ also belongs to class $\B$.
However, the converse is not true.
Define
\bd
\phi(r) = \left\{ \ba{ll} r, & \mbox{if } r \in [0,1] , \\
e^{-(r-1)}, & \mbox{if } r > 1 . \ea \right.
\ed
Then $\phi$ belongs to Class $\B$.
However, since $\phi(r) \ap 0$ as $r \ap \infty$, $\phi$ cannot be
bounded below by any function of class $\K$.
\end{example}

Suppose we wish to find a solution of $\f(\bth) = \bz$.
The convergence analysis of synchronous SA depends on the
stability of an associated ODE $\dot{\bth} = \f(\bth)$.
We now state a new theorem on global asymptotic stability, and then
use this to establish the convergence of the synchronous SA algorithm.
In order to state this theorem, we first introduce some standing
assumptions on $\f(\cdot)$.
Note that these assumptions are standard in the literature.
\ben
\item[(F1)] The equation $\f(\bth) = \bz$ has a unique solution $\bths$.
\item[(F2)] The function $\f$ is globally Lipschitz-continuous with
constant $L$.
\be\label{eq:416}
\nmeu{\f(\bth) - \f(\bphi)} \leq L \nmeu{\bth - \bphi} ,
\fa \bth , \bphi \in \R^d .
\ee
\een

\begin{theorem}\label{thm:41}
Suppose Assumption (F1) holds, and
that there exists a function $V: \R^d \ap \R_+$
and functions $\eta , \psi \in \KR, \phi \in \B$ such that
\be\label{eq:417}
\eta(\nmeu{\bth-\bths}) \leq V(\bth) \leq \psi(\nmeu{\bth-\bths})  ,
\fa \bth \in \R^d ,
\ee
\be\label{eq:418}
\Vd(\bth) \leq - \phi(\nmeu{\bth-\bths}) , \fa \bth \in \R^d ,
\ee
Then $\bths$ is a globally asymptotically stable equilibrium of the ODE
$\dot{\bth} = \f(\bth)$.
\end{theorem}

This is \cite[Theorem 4]{MV-GES-SA-arxiv22}, and the proof can be found
therein.
Well-known classical theorems for global asymptotic stability,
such as those found in \cite{Hahn67,MV-93,Khalil02}, require
the function $\phi(\cdot)$ to belong to Class $\K$.
Theorem \ref{thm:41} is an improvement,
in that the function $\phi(\cdot)$ is required only to belong
to the larger Class $\B$.

\subsection{A Convergence Theorem for Synchronous Stochastic Approximation}
\label{ssec:44}

In this subsection
we present a convergence theorem for synchronous stochastic approximation.
Theorem \ref{thm:42} below is sightly more general than a corresponding
result in \cite{MV-GES-SA-arxiv22}.
This theorem is obtained by combining some results from
\cite{MV-GES-SA-arxiv22} and \cite{MV-BASA-arxiv22}.
Other convergence theorems and examples can be found in \cite{MV-GES-SA-arxiv22}.

In order to analyze the convergence of the SA algorithm, we need to make
some assumptions about the nature of the measurement error sequence
$\{ \bxi_t \}$.
These assumptions are couched in terms of the conditional expectation
of a random variable with respect to a $\s$-algebra.
Readers who are unfamiliar with the concept are referred to
\cite{Durrett19} for the relevant background.

Let $\bth_0^t$ denote the tuple $\bth_0 , \bth_1 , \cdots , \bth_t$,
and define $\bxi_1^t$ analogously; note that there is no $\bxi_0$.
Let $\{ \F_t \}_{t \geq 0}$ be any filtration (i.e., increasing sequence
of $\s$-algebras), such that $\bth_0^t, \bxi_1^t$ are measurable with
respect to $\F_t$.
For example, one can choose 
$\F_t$ to be the $\s$-algebra generated by the tuples
$\bth_0^t,\bxi_1^t$.
\ben
\item[(N1)]
There exists a sequence $\{ b_t \}$ of nonnegative numbers such that
\be\label{eq:419}
\nmeu { E( \bxi_{t+1} | \F_t ) } \leq b_t \as, \fa t \geq 0 .
\ee
Thus $b_t$ provides a bound on the Euclidean norm of the conditional
expectation of the measurement error with respect to the $\s$-algebra $\F_t$.
\item[(N2)]
There exists a sequence $\{ \s_t \}$ of nonnegative numbers such that
\be\label{eq:4110}
E( \nmeusq{\bxi_{t+1} - E( \bxi_{t+1} | \F_t ) } | \F_t ) \leq \s_t^2 (1 + \nmeusq{\bth_t} ) ,
\as \fa t \geq 0 .
\ee
\een
Note that the quantity on the left side of \eqref{eq:4110} is the 
conditional variance of $\bxi_{t+1}$ with respect to the $\s$-algebra $\F_t$.

Now we can state a theorem about the convergence of synchronous SA.

\begin{theorem}\label{thm:42}
Suppose $\f(\bths) = \bz$, and Assumptions (F1--F2) and (N1--N2) hold.
Suppose in addition that there exists a $\C^2$
Lyapunov function $V: \R^d \ap \R_+$
that satisfies the following conditions:
\bit
\item There exist constants $a , b > 0$ such that
\be\label{eq:4113}
a \nmeusq{\bth-\bths} \leq V(\bth) \leq b \nmeusq{\bth-\bths} ,
\fa \bth \in \R^d .
\ee
\item There is a finite constant $M$ such that
\be\label{eq:4114}
\nmS{\nabla^2 V(\bth)} \leq 2M , \fa \bth \in \R^d .
\ee
\eit
With these hypothesis, we can state the following conclusions:
\ben
\item If $\Vd(\bth) \leq 0$ for all $\bth \in \R^d$,
and if
\be\label{eq:4111}
\sum_{t=0}^\infty \al_t^2 < \infty ,
\sum_{t=0}^\infty \al_t b_t < \infty ,
\sum_{t=0}^\infty \al_t^2 \s_t^2 < \infty ,
\ee
then the iterations $\{ \bth_t \}$ are bounded almost surely.
\item Suppose further that there exists a function $\phi \in \B$ such that
\be\label{eq:4115}
\Vd(\bth) \leq - \phi(\nmeu{\bth-\bths}) , \fa \bth \in \R^d .
\ee
and in addition to \eqref{eq:4111},
we also have
\be\label{eq:4112}
\sum_{t=0}^\infty \al_t = \infty ,
\ee
Then $\bth_t \ap \bths$ almost surely as $\tai$.
\een
\end{theorem}

Observe the nice ``division of labor'' between the two conditions:
Equation \eqref{eq:4111} guarantees the almost sure boundedness
of the iterations, while the addition of \eqref{eq:4112} leads to
the almost sure convergence of the iterations to the desired limit,
namely the solution of $\f(\bth) = \bz$.
This division of labor is first found in \cite{Gladyshev65}.
Theorem \ref{thm:42} is a substantial improvement on
\cite{Borkar-Meyn00}, which were the previously best results.
The interested reader is referred to \cite{MV-GES-SA-arxiv22} for
further details.

Theorem \ref{thm:42}
is a slight generalization of \cite[Theorem 5]{MV-GES-SA-arxiv22}.
In that theorem, it is assumed that $b_t = 0$ for all $t$, and
that the constants $\s_t$ are uniformly bounded by some constant $\s$.
In this case \eqref{eq:4111} and \eqref{eq:4112} become
\be\label{eq:4113}
\sum_{t=0}^\infty \al_t^2 < \infty ,
\sum_{t=0}^\infty \al_t = \infty .
\ee
These two conditions are usually referred to as the Robbins-Monro conditions.

\subsection{Convergence of Batch Asynchronous Stochastic Approximation}
\label{ssec:45}

Equations \eqref{eq:412} through \eqref{eq:414} represent what might be
called \textbf{Synchronous SA}, because at each time $t$, \textit{every}
component of $\bth_t$ is updated.
Variants of synchronous SA include Asynchronous SA (ASA), where at each time
$t$, exactly one component of $\bth_t$ is updated, and Batch
Asynchronous SA (BASA), where at each time $t$, some but not necessarily all
components of $\bth_t$ are updated.
We present the results for BASA, because ASA is a special case of BASA.
Moreover, we focus on \eqref{eq:413},
where the objective is to find a fixed point of a contractive map $\gbold$.
The modifications required for \eqref{eq:412} and \eqref{eq:414}
are straight-forward.

The relevant reference for these results is \cite{MV-BASA-arxiv22}.
As a slight modification of \eqref{eq:411}, it is assumed that, at each time
$t+1$, there is available a noisy measurement 
\be\label{eq:426}
\y_{t+1} = \gbold(\bth_t) - \bth_t + \bxi_{t+1} .
\ee
We assume that there is a given \textit{deterministic} sequence of
``step sizes'' $\{ \beta_t \}$.
In BASA, not every component of $\bth_t$ is updated at time $t$.
To determine which components are to be updated, we define $d$ different
binary ``update processes'' $\{ \kappa_{t,i} \}$, $i \in [d]$.
No assumptions are made regarding their independence.
At time $t$, define
\be\label{eq:421}
S(t) := \{ i \in [d] : \kappa_{t,i} = 1 \} .
\ee
This means that
\be\label{eq:422}
\th_{t+1,i} = \th_{t,i} , \fa i \not\in S(t) .
\ee
In order to define $\th_{t+1,i}$ when $i \in S(t)$, we make a distinction
between two different approaches: global clocks and local clocks.
If a global clock is used, then
\be\label{eq:423}
\al_{t,i} = \beta_t , \fa i \in S(t) ,
\al_{t,i} = 0 , \fa i \not\in S(t) .
\ee
If a local clock is used, then we first define the local counter
\be\label{eq:424}
\nu_{t,i} = \sum_{\t = 0}^t \kappa_{\t,i} , i \in [d] ,
\ee
which is the total number of occasions when $i \in S(\t)$, $0 \leq \t \leq t$.
Equivalently, $\nu_{t,i}$ is the total number of times up to and including
time $t$ when $\th_{\t,i}$ is updated.
With this convention, we define
\be\label{eq:425}
\al_{t,i} = \beta_{\nu_{t,i}} , \fa i \in S(t) ,
\al_{t,i} = 0 , \fa i \not\in S(t) .
\ee
The distinction between global clocks and local clocks was apparently
introduced in \cite{Borkar98}.
Traditional RL algorithms such as $\TDl$ and $Q$-learning, discussed in
detail in Section \ref{ssec:33} and again in
Sections \ref{ssec:52} and \ref{ssec:53}, use a global clock.
That is not surprising because \cite{Borkar98} came after
\cite{Sutton88} and \cite{Watkins-Dayan92}.
It is shown in \cite{MV-BASA-arxiv22} that the use of local clocks actually
simplifies the analysis of these algorithms.

Now we present the BASA updating rules.
Let us define the ``step size vector'' $\balpha_t \in \R_+^d$ via
\eqref{eq:423} or \eqref{eq:425} as appropriate.
Then the update rule is
\be\label{eq:427}
\bth_{t+1} = \bth_t + \balpha_t \circ \y_{t+1} ,
\ee
where $\y_{t+1}$ is defined in \eqref{eq:426}.
Here, the symbol $\circ$ denotes the Hadamard product of two vectors of
equal dimensions.
Thus if $\a, \b$ have the same dimensions, then $\c = \a \circ \b$
is defined by $c_i = a_i b_i$ for all $i$.

Recall that we are given a function $\gbold: \R^d \ap \R^d$, and the objective
is to find a solution to the fixed-point equation $\gbold(\bth) = \bth$.
Towards this end, we begin by stating the assumptions about the noise sequence.
\ben
\item[(N1')]
There exists a sequence of constants
$\{ b_t \}$ such that
\be\label{eq:428}
E( \nmeu{\bxi_{t+1}} | \F_t ) \leq b_t (1 + \nmi{\bth_0^t}) ,
\fa t \geq 0 .
\ee
\item[(N2')] There exists a sequence of constants $\{ \s_t \}$ such that
\be\label{eq:429}
E( \nmeusq{ \bxi_{t+1} - E( \bxi_{t+1} | \F_t ) } | \F_t )
\leq \s_t^2 ( 1 + \nmi{\bth_0^t}^2 ) , \fa t \geq 0 .
\ee
\een
Comparing \eqref{eq:419} and \eqref{eq:4110} with \eqref{eq:428} and
\eqref{eq:429} respectively, we see that the term $\nmeusq{\bth_t}$ is
replaced by $\nmi{\bth_0^t}$.
So the constants $b_t$ and $\s_t$ can be different in the two cases.
But because the two formulations is quite similar, we denote the first
set of conditions as (N1) and (N2), and the second set of conditions
as (N1') and (N2').

Next we state conditions on the step size sequence, which allow us to
state the theorems in a compact manner.
Next, we state the assumptions on the step size sequence.
Note that, if a local clock is used, then $\al_{t,i}$ can be
random even if $\beta_t$ is deterministic.
\ben
\item[(S1)]
The random step size sequences $\{ \al_{t,i} \}$ and the sequences
$\{b_t\}$, $\{\sigma^2_t\}$ and satisfy
\be\label{eq:4210}
\sum_{t=0}^\infty \al_{t,i}^2  < \infty,
\sum_{t=0}^\infty \sigma_t^2\al_{t,i}^2 < \infty,
\sum_{t=0}^\infty b_t\al_{t,i}  < \infty, \as,
\fa i \in [d] .
\ee
\item[(S2)] The random step size sequence $\{ \al_{t,i} \}$ satisfies
\be\label{eq:4211}
\sum_{t=0}^\infty \al_{t,i}  = \infty, \as,
\fa i \in [d] .
\ee
\een
Finally we state an assumption about the map $\gbold$.
\ben
\item[(G)] $\gbold$ is a contraction with respect to the $\ell_\infty$-norm
with some contraction constant $\g < 1$.
\een


\begin{theorem}\label{thm:43}
Suppose that Assumptions (N1') and (N2') about the noise sequence, (S1)
about the step size sequence, and (G) about the function $\gbold$ hold.
Then $\sup_t\nmi{\bth_t}<\infty$ almost surely.
\end{theorem}

\begin{theorem}\label{thm:44}
Let $\bths$ denote the unique fixed point of $\gbold$.
Suppose that Assumptions (N1') and (N2') about the noise sequence,
(S1) and (S2) about the step size sequence,
and (G) about the function $\gbold$ hold.
Then $\bth_t $ converges almost surely to $\bths$ as $\tai$.
\end{theorem}

The proofs of these theorems can be found in \cite{MV-BASA-arxiv22}.

\section{Applications to Reinforcement Learning}\label{sec:Appl}

In this section, we apply the contents of the previous section
to derive sufficient conditions for two distinct RL algorithms,
namely Temporal Difference Learning (without function approximation),
and $Q$-Learning.
Previously known results are stated in Section \ref{ssec:33}.
So what is the need to re-analyze those algorithms again from the
standpoint of stochastic approximation?
There are two reasons for doing so.
First, the historical $\TDl$ and $Q$-Learning algorithms are stated
using a ``global clock'' as defined in Section \ref{ssec:45}.
Subsequently, the concept of a ``local clock'' is introduced in
\cite{Borkar98}.
In \cite{MV-BASA-arxiv22}, the authors build upon this distinction
to achieve two objectives.
First, when a local clock is used, there are fewer assumptions.
Second, by proving a result on the sample paths of an irreducible
Markov process (proved in \cite{MV-BASA-arxiv22}), probabilistic
conditions such as \eqref{eq:325a}--\eqref{eq:325b} and
\eqref{eq:3313}--\eqref{eq:3314} are replaced by purely algebraic conditions.

\subsection{A Useful Theorem About Irreducible Markov Processes}
\label{ssec:51}

\begin{theorem}\label{thm:531}
Suppose $\{ N(t) \}$ is a Markov process on $[d]$ with a state
transition matrix $A$ that is irreducible.
Suppose $\{ \beta_t \}_{t \geq 0}$ is a sequence of real numbers in
$(0,1)$ such that $\beta_{t+1} \leq \beta_t$ for all $t$, and
\be\label{eq:5314}
\sum_{t=0}^\infty \beta_t = \infty .
\ee
Then
\be\label{eq:5315}
\sum_{t=0}^\infty \beta_t I_{\{ N(t) = i \} }(\om)  =
\sum_{t=0}^\infty \beta_t f_i(N(t)(\om)) = \infty , \fa i \in [d],
\fa \om \in \OM_0 ,
\ee
where $I$ denotes the indicator function.
\end{theorem}

\subsection{$TD$--Learning Without Function Approximation}\label{ssec:52}

Recall the $\TDl$ algorithm without function approximation,
presented in Section \ref{ssec:33}.
One observes a time series $\{ (X_t,R(X_t)) \}$
where $\{ X_t \}$ is a Markov process over $\X = \{ x_1 , \cdots , x_n \}$
with a (possibly unknown) state transition matrix $A$, and 
$R: \X \ap \R$ is a known reward function.
With the sample path $\{ X_t \}$ of the Markov process,
one can associate a corresponding ``index process''
$\{ N_t \}$ taking values in $[n]$, as follows:
\bd
N_t = i \mbox{ if } X_t = x_i \in \X .
\ed
It is obvious that the index process has the same transition matrix
$A$ as the process $\{ X_t \}$.
The idea is to start with an initial estimate $\vh_0$, and update it
at each time $t$ based on the sample path $\{ (X_t,R_t) \}$.

Now we recall the $\TDl$ algorithm without function approximation.
At time $t$, let $\vbh_t \in \R^n$ denote the current estimate of $\v$.
Let $\{ N_t \}$ be the index process defined above.
Define the ``temporal difference''
\be\label{eq:511}
\d_{t+1} := R_{N_t} + \g \Vh_{t,N_{t+1}} - \Vh_{t,N_t} , \fa t \geq 0 ,
\ee
where $\Vh_{t,N_t}$ denotes the $N_t$-th component of the vector $\vbh_t$.
Equivalently, if the state at time $t$ is $x_i \in \X$ and the
state at the next time $t+1$ is $x_j$, then
\be\label{eq:512}
\d_{t+1} = R_i + \g \Vh_{t,j} - \Vh_{t,i} .
\ee
Next, choose a number $\l \in [0,1)$.
Define the ``eligibility vector''
\be\label{eq:513}
\z_t = \sum_{\t = 0}^t (\g \l)^\t I_{ \{ N_{t-\t} = N_t \} } \eb_{N_{t-\t}} ,
\ee
where $\eb_{N_s}$ is a unit vector with a $1$ in location $N_s$
and zeros elsewhere.
Finally, update the estimate $\vbh_t$ as
\be\label{eq:515}
\vbh_{t+1} = \vbh_t + \d_{t+1} \al_t \z_t ,
\ee
where $\al_t$ is the step size chosen in accordance with either a global
or a local clock.
The distinction between the two is described next.
%

To complete the problem specification, we need to specify how the step
size  $\alpha_t$ is chosen in \eqref{eq:515}.
The two possibilities studied here are: global clocks and local clocks.
If a global clock is used, then $\al_t = \beta_t$, whereas if
a local clock is used, then $\al_t = \beta_{\nu_{t,i}}$, where
\bd
\nu_{t,i} = \sum_{\t=0}^t I_{ \{ z_{\t,i} \neq 0 \} } .
\ed
Note that in the traditional implementation of the $\TDl$ algorithm
suggested in \cite{Sutton88,Tsi-Van-TAC97,Jaakkola-et-al94},
a global clock is used.
Moreover, the algorithm is shown to converge provided
\be\label{eq:516}
\sum_{t=0}^\infty \al_t^2 < \infty , \sum_{t=0}^\infty \al_t = \infty , \as
\ee
As we shall see below, the theorem statements when local clocks are used
involve slightly fewer assumptions than when global clocks are used.
Moreover, neither involves probabilistic conditions such as
\eqref{eq:325a} and \eqref{eq:325b}, in contrast to Theorem \ref{thm:TDl}.

Next we present two theorems regarding the convergence of the $TD(0)$
algorithm.
As the hypotheses are slightly different, they are presented separately.
But the proofs are quite similar, and can be found in
\cite{MV-BASA-arxiv22}.

\begin{theorem}\label{thm:51}
Consider the $\TDl$ algorithm using a local clock to determine the
step size.
Suppose that the state transition matrix $A$ is irreducible, and that
the deterministic step size sequence $\{ \beta_t \}$ satisfies the
Robbins-Monro conditions
\bd
\sum_{t=0}^\infty \beta_t = \infty , \sum_{t=0}^\infty \beta_t^2 < \infty .
\ed
Then $\v_t \ap \v$ almost surely as $\tai$.
\end{theorem}

\begin{theorem}\label{thm:52}
Consider the $\TDl$ algorithm using a global clock to determine the
step size.
Suppose that the state transition matrix $A$ is irreducible,
and that the deterministic step size sequence is nonincreasing
(i.e., $\beta_{t+1} \leq \beta_t$ for all $t$), and
satisfies the Robbins-Monro conditions as described above.
Then $\v_t \ap \v$ almost surely as $\tai$.
\end{theorem}

\subsection{$Q$-Learning}\label{ssec:53}

The $Q$-learning algorithm proposed in \cite{Watkins-Dayan92} 
is now recalled for the convenience of the reader.
\ben
\item Choose an arbitrary initial guess $Q_0 : \X \times \U \ap \R$
and an initial state $X_0 \in \X$.
\item At time $t$, with current state $X_t = x_i$, choose a current action
$U_t = u_k \in \U$, and let the Markov process run for one time step.
Observe
the resulting next state $X_{t+1} = x_j$.
Then update the function $Q_t$ as follows:
\be\label{eq:61}
\begin{split}
Q_{t+1}(x_i,u_k) & =
Q_t(x_i,u_k) + \beta_t [ R(x_i,u_k) + \g V_t(x_j) - Q_t(x_i,u_k) ] , \\
Q_{t+1}(x_s,w_l) & = Q_t(x_s,w_l) , \fa (x_s,w_l) \neq (x_i,u_k) .
\end{split}
\ee
where
\be\label{eq:62}
V_t(x_j) = \max_{w_l \in \U} Q_t(x_j,w_l) ,
\ee
and $\{ \beta_t \}$ is a deterministic sequence of step sizes.
\item Repeat.
\een
In earlier work such as \cite{Tsi-ML94,Jaakkola-et-al94}, it is shown that
the $Q$-learning algorithm converges to the optimal action-value
function $Q^*$ \textit{provided}
\be\label{eq:63}
\sum_{t=0}^\infty \beta_t I_{(X_t,U_t) = (x_i,u_k)} = \infty,
\fa (x_i,u_k) \in \X \times \U ,
\ee
\be\label{eq:64}
\sum_{t=0}^\infty \beta_t^2 I_{(X_t,U_t) = (x_i,u_k)} < \infty,
\fa (x_i,u_k) \in \X \times \U .
\ee
These conditions are stated here as Theorem \ref{thm:QL}.
Similar hypotheses are present in all existing results in asynchronous SA.
Note that in the $Q$-learning algorithm, there is no guidance on how
to choose the next action $U_t$.
Presumably $U_t$ is chosen so as to ensure that \eqref{eq:63} and
\eqref{eq:64} are satisfied.
However, we now demonstrate a way to avoid such conditions, by using
Theorem \ref{thm:531}.
We also introduce batch updating
and show that it is possible to use a local clock instead of a global clock.

The batch $Q$-learning algorithm introduced here is as follows:
\ben
\item Choose an arbitrary initial guess $Q_0 : \X \times \U \ap \R$,
and $m$ initial states $X^k_0 \in \X, k \in [m]$, in some fashion
(deterministic or random).
Note that the $m$ initial states need not be distinct.
\item At time $t$, for each action index $k \in [m]$,
with current state $X^k_t = x^k_i$, choose the current action
as $U_t = u_k \in \U$, and let the Markov process run for one time step.
Observe the resulting next state $X^k_{t+1} = x^k_j$.
Then update function $Q_t$ as follows, once for each $k \in [m]$:
\be\label{eq:65}
Q_{t+1}(x^k_i,u_k) =
\left\{ \ba{ll}
Q_t(x^k_i,u_k) + \al_{t,i,k} [ R(x_i,u_k) + \g V_t(x^k_j) - Q_t(x^k_i,u_k) ] ,
& \mbox{if } x_s = x^k_i , \\
Q_t(x^k_s,u_k) , & \mbox{if } x^k_s \neq x^k_i .
\ea \right.
\ee
where
\be\label{eq:66}
V_t(x^k_j) = \max_{w_l \in \U} Q_t(x^k_j,w_l) .
\ee
Here $\al_{t,i,k}$ equals $\beta_t$ for all $i,k$ if a global clock is
used, and equals
\be\label{eq:67}
\al_{t,i,k} = \sum_{\t = 0}^t I_{ \{ X^k_t = x_i \} } 
\ee
if a local clock is used.
\item Repeat.
\een

\textbf{Remark:}
Note that $m$ different simulations are being run in parallel, and that
in the $k$-th simulation, the next action $U_t$ is always chosen as $u_k$.
Hence, at each instant of time $t$, exactly $m$ components of $Q(\cdot,\cdot)$
(viewed as an $n \times m$ matrix) are updated, namely the $(X^k_t,u_k)$
component, for each $k \in [m]$.
In typical
MDPs, the size of the action space $m$ is much smaller than the size of
the state space $n$.
For example, in the Blackjack problem discussed in \cite[Chapter 4]
{Sutton-Barto18}, $n \sim 2^{100}$ while $m = 2$!
Therefore the proposed batch $Q$-learning algorithm is quite efficient
in practice.

Now, by fitting this algorithm into the framework of Theorem
\ref{thm:531}, we can prove the following general result.
The proof can be found in \cite{MV-BASA-arxiv22}.

\begin{theorem}\label{thm:61}
Suppose that each matrix $A^{u_k}$ is irreducible, and that the
step size sequence $\{ \beta_t \}$ satisfies the
Robbins-Monro conditions \eqref{eq:4113} with $\al_t$ replaced by $\beta_t$.
With this assumption, we have the following:
\ben
\item If a local clock is used as in \eqref{eq:65},
then $Q_t$ converges almost surely to $Q^*$.
\item If a global clock is used (i.e., $\al_{t,i,k} = \beta_t$
for all $t,i,k$), and $\{ \beta_t \}$ is nonincreasing,
then $Q_t$ converges almost surely to $Q^*$.
\een
\end{theorem}

\textbf{Remark:}
Note that, in the statement of the theorem, it is \textit{not} assumed
that every \textit{policy} $\pi$ leads to an irreducible
Markov process -- only that every \textit{action} leads to an
an irreducible Markov process.
In other words, the assumption  is
that the $m$ different matrices $A^{u_k}, k \in [m]$ correspond to
irreducible Markov processes.
This is a substantial improvement.
It is shown in \cite{Tsi-ORL07} that the following problem is NP-hard:
Given an MDP, determine whether \textit{every policy} $\pi$ results in
a Markov process that is a unichain, that is, consists of a single set of
recurrent states with the associated state transition matrix being
irreducible, plus possibly some transient states.
Our problem is slightly different, because we don't permit any transient states.
Nevertheless, this problem is also likely to be very difficult.
By not requiring any condition of this sort, and also by dispensing
with conditions analogous to \eqref{eq:63} and \eqref{eq:64}, the
above theorem statement is more useful.


\section{Conclusions and Problems for Future Research}\label{sec:Future}

In this brief survey, we have attempted to sketch some of the highlights of
Reinforcement Learning.
Our viewpoint, which is quite mainstream, is to view RL as solving
Markov Decision Problems (MDPs) when the underlying dynamics are unknown.
We have used the paradigm of Stochastic Approximation (SA) as a unifying
approach.
We have presented convergence theorems for the standard approach, which
might be thought of as ``synchronous'' SA, as well as variants such as
Asynchronous SA (ASA) and Batch Asynchronous SA (BASA).
Many of these results are due to the author and his collaborators.

In this survey, due to length limitations,
we have \textit{not} discussed actor-critic algorithms.
These can be viewed as applications of the policy gradient theorem
\cite{Sutton-PG00,Mar-Tsi-TAC01} coupled with stochastic approximation
applied to two-time scale (i.e., singularly perturbed) systems
\cite{Borkar97,CL-SB-Auto17}.
Some other relevant references are
\cite{Konda-Borkar99,Konda-Tsi99,Konda-Tsi03}.
Also, the rapidly emerging field of Finite-Time SA has not been discussed.
FTSA can lead to estimates of the rate of convergence of various RL
algorithms, whereas conventional SA leads to only asymptotic results.
Some recent relevant papers include \cite{CMSS20-NeurIPS20,CMSS-NeurIPS21}.


\end{document}